\documentclass{article}
\usepackage[utf8]{inputenc}
\usepackage[margin=1.0in]{geometry}
\pdfoutput=1 
\usepackage{graphicx}		
\usepackage{amssymb}
\usepackage{amsmath}
\usepackage{amsthm}
\usepackage{mathtools}
\usepackage{hyperref}
\usepackage{enumitem}
\usepackage{verbatim}
\usepackage{xcolor}
\usepackage[T1]{fontenc}
\usepackage{microtype}
\usepackage{natbib}
\usepackage[bottom]{footmisc}
\usepackage{bbm}
\usepackage{float}
\usepackage{booktabs}
\usepackage[acronym]{glossaries}

\newcommand\blfootnote[1]{%
  \begingroup
  \renewcommand\thefootnote{}\footnote{#1}%
  \addtocounter{footnote}{-1}%
  \endgroup
}

\newtheorem{theorem}{Theorem}
\newtheorem{definition}{Definition}
\newtheorem{proposition}{Proposition}
\newtheorem{lemma}{Lemma}

\newtheorem{corollary}{Corollary}

  \newtheorem{example}{Example}

\renewcommand{\iff}{\Leftrightarrow}

\newcommand{\ex}[2]{{\ifx&#1& \mathbb{E} \else \underset{#1}{\mathbb{E}} \fi \left[#2\right]}}
\newcommand{\exinline}[2]{{\ifx&#1& \mathbb{E} \else {\mathbb{E}}_{#1} \fi \left[#2\right]}}
\newcommand{\pr}[2]{{\ifx&#1& \mathbb{P} \else \underset{#1}{\mathbb{P}} \fi \left[#2\right]}}
\newcommand{\var}[2]{{\ifx&#1& \mathsf{Var} \else \underset{#1}{\mathsf{Var}} \fi \left[#2\right]}}

\newcommand{\Ber}{\textup{Ber}}
\newcommand{\R}{\mathbb{R}}

\DeclareMathOperator*{\argmin}{arg\,min}

\newcommand{\cD}{\mathcal{D}}
\newcommand{\cF}{\mathcal{F}}
\newcommand{\cP}{\mathcal{P}}
\newcommand{\cQ}{\mathcal{Q}}
\newcommand{\cT}{\mathcal{T}}
\newcommand{\cX}{\mathcal{X}}
\newcommand{\cY}{\mathcal{Y}}
\newcommand{\cZ}{\mathcal{Z}}
\newcommand{\cG}{\mathcal{G}}
\newcommand{\N}{\mathbb{N}}
\newcommand{\llog}{\ensuremath{\text{\rm llog\ }}}
\newcommand{\etamax}{\eta_{\max}}
\newcommand{\etamin}{\eta_{\min}}
\newcommand{\estimate}[1]{\hat{#1}}
\newcommand{\minim}[2]{{\left( #1 \wedge #2 \right)}}

\newcommand{\sZ}{\tilde{Z}} 
\newcommand{\sz}{\tilde{z}}
\newcommand{\sx}{\tilde{x}}

\newcommand{\dkl}[2]{\mathsf{KL}\left(#1\middle\|#2\right)}
\newcommand{\kl}{\mathsf{KL}}

\newcommand{\CMI}[2]{{\ifx&#2& \mathsf{CMI} \else \mathsf{CMI}_{#2} \fi \left(#1\right)}}

\let\originalleft\left
\let\originalright\right
\renewcommand{\left}{\mathopen{}\mathclose\bgroup\originalleft}
\renewcommand{\right}{\aftergroup\egroup\originalright}

\def\comments{1}
\ifnum\comments=1
\newcommand{\mybignote}[2]{{\color{#1} [[#2]]}}
\else
\newcommand{\mybignote}[2]{}
\fi
\newcommand{\lnote}[1]{\mybignote{teal}{L: #1}}

\newcommand{\stochleq}{\ensuremath{\trianglelefteq}}
\newcommand{\commentout}[1]{}

\DeclareRobustCommand{\VANDER}[3]{#2}

\title{PAC-Bayes, MAC-Bayes and Conditional Mutual Information: Fast rate bounds that handle general VC classes}
\author{Peter Gr\"unwald\thanks{CWI Amsterdam and Leiden University. \dotfill \texttt{peter.grunwald@cwi.nl}} \and Thomas Steinke\thanks{Google Research, Brain Team. \dotfill \texttt{fast@thomas-steinke.net}} \and Lydia Zakynthinou\thanks{Khoury College of Computer Sciences, Northeastern University. \dotfill \texttt{zakynthinou.l@northeastern.edu}}}

\begin{document}

\maketitle
\begin{abstract}
    We give a novel, unified derivation of {\em conditional\/} PAC-Bayesian and mutual information (MI) generalization bounds. We derive conditional MI bounds as an instance, with special choice of prior,  of conditional {\em MAC}-Bayesian (Mean Approximately Correct) bounds, itself derived from conditional PAC-Bayesian bounds, where `conditional' means that one can use priors conditioned on a joint training and ghost sample.  This allows us to get nontrivial PAC-Bayes and MI-style bounds for general VC classes, something recently shown to be  impossible with standard PAC-Bayesian/MI bounds. Second, it allows us to get faster rates of order $O \left(({\kl}/n)^{\gamma}\right)$ for $\gamma > 1/2$ if a Bernstein condition holds and for exp-concave losses (with $\gamma=1$), which is impossible with both standard PAC-Bayes generalization and MI bounds. Our work extends the recent work by~\citet{SteinkeZ20} who handle MI with VC but neither PAC-Bayes nor fast rates, the recent work of~\citet{HellstromD20} who extend the latter to the PAC-Bayes setting via a unifying exponential inequality, and~\citet{MhammediGG19} who initiated fast rate PAC-Bayes generalization error bounds but handle neither MI nor general VC classes. 
\end{abstract}

\blfootnote{Accepted for publication at the 34th Annual Conference on Learning Theory (COLT 2021).}

\section{Extended Introduction}\label{sec:intro}
We first give a mini-introduction to PAC-Bayesian and mutual information bounds. Then we indicate two deficiencies of such bounds and give an informal statement of our main result, which solves both issues for both types of bounds at the same time. At the end of the introduction we discuss related work. In the remaining sections~\ref{sec:preliminaries}--\ref{sec:conclusion}, we provide additional mathematical preliminaries, then we state our main lemma (proof delegated to an appendix) and use it to prove our main theorem and present its implications.
\paragraph{Setting}
In the standard setting of supervised learning, we are given a {\em model}, i.e., a set  $\cF$, where each $f\in\cF$ is a hypothesis that takes the form of a {\em predictor}. Our aim is to learn to predict well based on a sample of $n$ i.i.d.\ examples $Z = (Z_1, \ldots, Z_n)$ drawn from an unknown distribution $\cD$ over the space of examples, $\cZ$. We will denote the random variable representing a sample by $Z$, whereas a single example will be denoted by a $Z_i$, as previously, or by $Z'$. We adopt the convention of using upper-case letters for random variables (RVs) and lower-case letters for their realizations. A {\em learning algorithm\/} $A: \cZ^n \rightarrow \Delta(\cF)$ (where $\Delta(\cF)$ is the set of distributions over $\cF$)  takes as input the sample $Z$  and outputs a distribution over hypotheses. The special case of deterministic predictors such as ERM is covered by allowing the algorithm to output distributions on a single $f \in \cF$. We  refer to the posterior distribution of the output of $A$ given input $Z$ by $A|Z$.
For a loss function $\ell:\cF\times\cZ \rightarrow \R$, $\ell(f;z')$ denotes the loss of a deterministic hypothesis $f\in\cF$ on an example $z'\in\cZ$. We extend this notation to define the true loss and the empirical loss of $f$ on a sample $z\in\cZ^n$ by $\ell(f;\cD)=\exinline{Z'\sim\cD}{\ell(f; Z')}$ and $\ell(f;z)=\frac{1}{n}\sum_{i=1}^n\ell(f;z_i)$, respectively. Furthermore, for a randomized hypothesis $F\in\Delta(\cF)$, we define the expected true loss and the empirical loss on sample $z\in\cZ^n$ by $L(F;\cD)=\exinline{f\sim F}{\ell(f;\cD)}$ and $L(F;z)=\exinline{f\sim F}{\ell(f;z)}$, respectively.  A {\em learning problem\/} is a tuple $(\cD, \ell, \cF)$.

\paragraph{Standard PAC-Bayesian bounds} Within this setting, a standard goal is to  bound the {\em generalization error\/} of an algorithm $A$ in terms of its {\em empirical/training error}. A standard way to achieve this, which has recently received renewed attention, are {\em PAC-Bayesian generalization error bounds} \citep{McAllester98,McAllester02,langford2003pac,Seeger02,maurer2004,Audibert04,Catoni07,ambroladze2007tighter} which commonly take the form: 
\begin{equation}\label{eq:firstgebound}
\overbrace{L(A| Z ; \cD)}^{\text{generalization error}} - \overbrace{L(A | Z ;  Z )}^{\text{training error}} 
\stochleq   C_1  \cdot  \sqrt{\frac{L(A | Z ;  Z ) \cdot \dkl{A | Z}{\pi}}{n}} + C_2 
\cdot   \frac{\dkl{A | Z}{\pi}}{n} 
\end{equation}
for some constants $C_1, C_2 > 0$ and $\dkl{A | Z}{\pi}$ being the $\kl$ divergence between the `posterior' output of the algorithm and the `prior' distribution $\pi$  over $\cF$. The bounds hold for arbitrary priors $\pi$, as long as these are chosen independently of the data $Z$. Here we are ignoring $O(\log n)$ factors.
The notation $\stochleq$ expresses that the equation holds up to a small additive term with high probability over the distribution $\cD^n$ of the training sample $Z$ as well as in expectation. To be precise, (\ref{eq:firstgebound}) holds as an {\em exponential stochastic inequality\/} or ESI (pronounced `easy'), a useful concept introduced and used by \citet{WooterGE16} and \citet{GrunwaldM20}, which we will use throughout this paper.
\begin{definition}[Exponential Stochastic Inequality (ESI) \citep{GrunwaldM20}]\label{def:esi}
Let $\eta>0$ and $X,Y$ be random variables that can be expressed as functions of the random variable $U$ defined on the probability space $\cD^n$. Then
\[X\stochleq_\eta^U Y \Leftrightarrow \ex{U}{e^{\eta(X-Y)}}\leq 1.\]
\end{definition}
When no ambiguity can arise, we omit the random variable $U$. Besides simplifying notation, ESIs are useful in that they  simultaneously capture ``with high probability'' and ``in expectation'' results, that is, $X \stochleq_{\eta}^U Y$, implies both that $\forall \delta \in(0,1)$, $X\leq Y + \log (1/\delta)/\eta$, with probability at least $1-\delta$ over the randomness of $U$ and that $\exinline{U}{X} \leq \exinline{U}{Y}$.

The standard PAC-Bayes bound~\eqref{eq:firstgebound} has recently been applied to practically important continuously parameterized model classes, such as deep neural networks \citep{dziugaite2017computing,zhou2018}. The prior then takes the form of a probability density over the parameters (e.g. weights $\vec{w}$) and in order for the $\kl$ term to be finite, one needs to randomize the output of the algorithm. Therefore, even if the empirical error of the output $\vec{w}|Z$ of the original learning algorithm (typically SGD) can be driven down to $0$, the empirical error as appearing in (\ref{eq:firstgebound}), and therefore also the multiplication factor inside the square root, is not $0$---one typically takes a Gaussian around $\vec{w}|Z$ leading to a nonnegligible $L(A | Z ;  Z )$ (\citet{MhammediGG19} provide a numerical example). 
\paragraph{Standard Mutual Information (MI) Bounds}
Another, related way to bound generalization error is provided by {\em mutual information bounds\/} \citep{RussoZ16,XuR17}. These usually take on the following form:
\begin{equation}\label{eq:firstmibound}
\left|\ex{Z}{L(A| Z ; \cD) - L(A|Z ; Z)}\right|
\le \sqrt{\frac{2\cdot I(A|Z ; Z)}{n}},
\end{equation}
with $I(A|Z ; Z)$ denoting the mutual information between the training data and the algorithm's output. 
\paragraph{Two Issues with the Bounds}
Standard PAC-Bayesian and MI bounds have two deficiencies in common. 
First, as recently shown by~\citet{LivniMoran20}, there exist hypothesis classes with finite Vapnik-Chervonenkis (VC) dimension $d$ for which, rather than achieving the standard VC generalization error bound of order  $\sqrt{(d \log n)}/{n}$, PAC-Bayes bounds of the form~\eqref{eq:firstgebound} must remain trivial: there exists a VC class, such that for any arbitrary learning algorithm $A$, there exists a realizable (i.e., $\inf_{f \in \cF} \ell(f; \cD) = 0$) distribution $\cD$, such that for any prior $\pi$ (even one that is allowed to depend on the data-generating distribution $\cD$), either the $\kl$ divergence term $\dkl{A|Z}{\pi}$ is arbitrarily large or the loss is large ($L(A|Z;\cD)>1/4$).
Similarly,~\citet{BassilyMNSY18} and \citet{NachumSY18} show that there exists a VC class such that, for any proper and consistent learning algorithm $A$, there exists a realizable distribution $\cD$, such that the mutual information $I(A|Z;Z)$ in the bound of~\eqref{eq:firstmibound} is arbitrarily large. 

Second, in both theoretically interesting settings (such as random label noise, see Example~\ref{ex:threshold} below) and in practical settings (as already indicated above) the empirical error term $L(A |Z ; Z )$ inside the square root of~\eqref{eq:firstgebound} often cannot be ignored. Then both bounds (\ref{eq:firstgebound}) and (\ref{eq:firstmibound}) will be of order $\sqrt{\textsc{complexity}/n}$. The theory of {\em excess risk bounds\/} suggests that this is, in many cases, suboptimal and we can obtain a more desirable bound of the form $\textsc{complexity}/n$. Here we concentrate on the following typical form of PAC-Bayesian excess risk bounds \citep{Audibert04,zhang2006epsilon,zhang2006information,GrunwaldM20,GrunwaldM19}, but the results are comparable in nature to excess risk bounds based on e.g. Rademacher complexity bounds \citep{bartlett2006empirical}:
\begin{align}\label{eq:firsterbound}
\overbrace{R(A |Z; \cD)}^{\text{excess risk}}
     \stochleq C_3 \cdot \overbrace{R(A |Z ; Z)}^{\text{empirical excess risk}}  + C_4  \cdot  \left(
    \frac{
    \dkl{A | Z}{\pi}}{n} \right)^{\gamma}
\end{align}
for some constants $C_3, C_4 > 1$ and $\gamma \in [1/2, 1]$.
Here we ignore $O(\log \log n)$ factors. The {\em excess risk\/}  of a distribution over predictors $F \in \Delta(\cF)$ is defined as $R(F;\cD)=L(F;\cD)-L(f^*;\cD)$  where $f^*$ is an 
optimal predictor within the class $\cF$, achieving $\min_{f \in \cF} \ell(f; \cD)$, whose existence is commonly assumed (e.g. \citet{Tsybakov04,bartlett2006empirical,GrunwaldM20}). The excess risk of algorithm $A$ based on training sample $Z$, $R(A | Z; \cD)$, is thus a nonnegative random variable (depending on $Z$) denoting the additional risk incurred if one predicts based on the learned distribution $A |Z$, compared to the best one could have with knowledge of the true distribution $\cD$. Similarly, the {\em empirical excess risk\/} of $F$ on a sample $z\in\cZ^n$ is  $R(F;z)=L(F;z)-L(f^*;z)$.
Substituting these terms and rearranging, inequality~\eqref{eq:firsterbound} can be written as follows, giving an upper bound on the generalization gap:
\begin{align}\label{eq:firsterboundtogebound}
L(A |Z; \cD)-L(A|Z;Z)
     \stochleq (L(f^*; \cD)-L(f^*;Z))+ (C_3-1) \cdot R(A |Z ; Z)  + C_4  \cdot  \left(
    \frac{
    \dkl{A | Z}{\pi}}{n} \right)^{\gamma}
\end{align}

The $\gamma$ for which (\ref{eq:firsterbound}) holds depends on the interplay between the model $\cF$, the loss function $\ell$, and the true distribution $\cD$.
Specifically, a sufficient condition for the result to hold for $\gamma = 1/(2-\beta)$ is if the learning problem $(\cD,\ell,\cF)$ satisfies a $(B,\beta)$-{\em Bernstein condition\/} \citep{BartlettBL02,bartlett2006empirical,vanerven2015fast}:
\begin{definition}[{Bernstein Condition}] \label{def:Bernstein}
Let $\beta\in[0,1]$ and $B\geq 1$. Then $(\cD, \ell, \cF)$ satisfies the \emph{$(B,\beta)$-Bernstein condition} if there exists a $f^*\in\cF$ such that
\begin{equation}\label{eq:bernstein}
\ex{Z'\sim\cD}{(\ell(f;Z')-\ell(f^*; Z'))^2} \leq B\left(\ex{Z'\sim\cD}{\ell(f;Z')-\ell(f^*; Z')}\right)^{\beta} ~ \text{ for all } f\in\cF.\end{equation}
\end{definition}
If the Bernstein condition (\ref{eq:bernstein}) holds for some $f^*$, then this $f^*$ must be an optimal predictor as above.
If the losses are assumed bounded then the Bernstein condition vacuously holds for $\beta=0$ with some $B$. Throughout this paper, the losses are assumed in $[0,1]$, hence it always holds with $\beta =0, B=4$. Therefore, the {\em slow rate\/} of $\gamma = 1/(2-0) = 1/2$ can always be obtained. But for loss functions with curvature (specifically, all bounded so-called {\em mixable\/} loss functions, which includes all {\em exp-concave} loss functions \citep{vanerven2015fast}), the Bernstein condition also holds with $\beta=1$, implying {\em fast\/} $O(1/n)$ rates, i.e., $\gamma = 1$. Examples include the bounded squared error loss and logistic loss. Specifically, for the squared loss $\ell(f;(X,Y)) := (Y- f(X))^2$ (rescaled so that all functions map $X$ to $[-1/2,1/2]$ and $Y \in [-1/2,1/2]$ so that the range is $[0,1]$) it automatically holds with $\beta=1$ and $B=4$~\citep[Proposition 19]{GrunwaldM20}. Even for the nonmixable $0/1$-loss,  a Bernstein condition may still hold. For example, in the realizable case and in the case of random label noise (homoskedasticity), the Massart condition and, hence, the Bernstein condition holds, giving $\gamma = 1$.
The Bernstein condition is a significant weakening of the perhaps more well-known Tsybakov-Mammen \citep{Tsybakov04} condition which itself is a weakening of the Massart condition for classification; see \citet{vanerven2015fast} for an extensive overview and links between a variety of ``easiness'' conditions such as (Massart, Bernstein and Tsybakov) proposed in the literature. \citet{Tsybakov04} provides examples of situations in which Bernstein holds for $\beta$ strictly between $0$ and $1$, where {\em faster/intermediate\/} rates can be obtained.

For many algorithms, the  empirical excess risk term $R(A | Z ; Z)$ will be negligible. For example, for ERM (Empirical Risk Minimization) it will automatically be nonpositive since by definition the ERM cannot have larger loss on the sample than $f^*$. In addition, the first term, that is, the excess risk of $f^*$, disappears when the inequality is weakened to an in-expectation bound, while introducing a small unavoidable term in the in-probability bound. Then, in many settings, the right-hand side  of (\ref{eq:firsterboundtogebound}) is clearly smaller than that of (\ref{eq:firstgebound}) which suggests that the standard generalization bound  (\ref{eq:firstgebound}) is suboptimal as soon as a Bernstein condition holds with $\beta > 0$. Below we shall see that this is indeed the case.

\paragraph{Solving Both Issues at Once for both Bounds}
Partial solutions for both issues were provided by \citet{Audibert04,Catoni07,MhammediGG19,SteinkeZ20, HellstromD20}. By combining their insights and adding a new fundamental lemma (Lemma~\ref{lem:esiindividual} below), we manage to solve both problems for both types of bounds 
in essentially a single derivation. Its first intermediate conclusion is the following {\em faster rate data-conditional generalization error bound\/} (Theorem~\ref{th:main} below):
Let $(\cD,\ell,\cF)$ represent a learning problem which satisfies the $(B,\beta)$-Bernstein condition and suppose the loss function $\ell$ is bounded. Let the data $\sZ_{\bf 0} = (\sZ_{1,0}, \ldots, \sZ_{n,0})^\top\in\cZ^n$ be i.i.d. $\sim \cD$. Then for arbitrary {\em almost exchangeable data-dependent priors} $\pi \mid \langle \sZ_{\bf 0}, \sZ_{\bf 1} \rangle$ we have: 
\begin{multline}\label{eq:mainpre} 
{L(A| \sZ_{\bf 0} ; \cD) - L(A | \sZ_{\bf 0} ;  \sZ_{\bf 0})} \stochleq   \minim{1}{2\beta} \cdot {R(A | \sZ_{\bf 0} ;  \sZ_{\bf 0})} +  O \left(\frac{\ex{\sZ_{\bf 1}}{\dkl{A | \sZ_{\bf 0}}{\pi| \langle \sZ_{\bf 0}, \sZ_{\bf 1} \rangle }} }{n}  \right)^{\frac{1}{2-\beta}} + \frac{6 \eta}{n}
\end{multline}
Here $\wedge$ denotes minimum, the result holds up to $\log \log n $ factors and it requires an additional condition which essentially holds as long as $\dkl{\cdot}{\cdot} = o(n)$ almost surely under $\cD$. Note that this is an ESI inequality and as such it holds both in expectation and up to a small additive term with high probability over the training sample $\sZ_{\bf 0}$. We return later to this fact and to the remainder term $6 \eta/n$, which for now may be thought of as negligible.

To appreciate (\ref{eq:mainpre}), first note that, since the Bernstein condition automatically holds for $\beta=0$, so does (\ref{eq:mainpre}). Then the first term on the right disappears and the $\kl$ term becomes of order $\sqrt{\kl/n}$, as is the leading term for classical PAC-Bayesian bounds. However, in stark contrast to classical PAC-Bayesian bounds, we are now allowed (not required) to use priors which can {\em depend on the data in many -- but not arbitrary -- ways}: just like in classical Vapnik-Chervonenkis learning theory, we imagine a {\em ghost sample\/} $\sZ_{\bf 1}$ of equal size and distribution as the training sample $\sZ_{\bf 0}$. The notation $$\langle \sZ_{\bf 0}, \sZ_{\bf 1} \rangle := (\{\sZ_{1,0}, \sZ_{1,1}\}, \{\sZ_{2,0}, \sZ_{2,1}\}, \ldots, \{\sZ_{n,0}, \sZ_{n,1}\})^\top $$ indicates a vector of $n$ {\em unordered\/} pairs of examples, where the $i$-th component is the bag of example $i$ in the training sample $\sZ_{\bf 0}$ and example $i$ in the ghost sample $\sZ_{\bf 1}$. The prior $\pi | \langle \sZ_{\bf 0} , \sZ_{\bf 1} \rangle$ is allowed to depend on these $2n$ examples that include all the $n$ training examples, but all information as to whether an example is in the training or ghost sample is hidden from the prior. The complexity is then measured as the {\em expected\/} $\kl$ divergence where the ghost sample is i.i.d.~$\sim \cD$. More formally, let us write $\sZ_S=(\sZ_{1,S_1}, \ldots, \sZ_{n,S_n})^\top\in\cZ^{n\times 1}$ for the sample  whose $i$-th example belongs to the sample $\sZ_{\bf 0}$ or $\sZ_{\bf 1}$, as indicated by $S_i\in\{0,1\}$ and let $\sZ_{\bar{S}}=(\sZ_{1,\bar{S}_1}, \ldots, \sZ_{n,\bar{S}_n})^\top$ be its complement.

\begin{definition}[{Almost Exchangeable Prior, terminology from~\citet{Audibert04}}]\label{def:exprior} A function (conditional distribution) $\pi: \cZ^{n \times 2} \rightarrow \Delta(\cF)$ is \emph{almost exchangeable} if for all $\sz\in\cZ^{n\times 2}$, it holds that $\pi|(\sz_s,\sz_{\bar{s}})=\pi|(\sz_{\bf 0}, \sz_{\bf 1})$, $\forall s\in\{0,1\}^n$, justifying the notation $\pi|\langle \sz_s,\sz_{\bar{s}}\rangle =\pi|\langle \sz_{\bf 0}, \sz_{\bf 1} \rangle$.
\end{definition}

It may appear that the expectation over the ghost sample makes such $\kl$ bounds incalculable in practice, but this is not so: in Section~\ref{sec:applications} we give examples of data-dependent almost exchangeable priors for which the $\kl$ complexity term, or at least a good upper bound, can be calculated based on the observed data. In particular, in classification with a class $\cF$ with finite VC dimension $d$, when an ERM algorithm with a specific consistency property is used (Theorem~\ref{th:vc-kl} shows that such an ERM can always be constructed), the $\kl$ term can be bounded as $d \log (2n)$, leading us to recover classical VC bounds; similarly, for size $k$-compression schemes, the $\kl$ term is also bounded as $k \log (2n)$.

Now suppose a Bernstein condition holds for some $\beta > 0$. We then see that (\ref{eq:mainpre}) gives a {\em faster rate bound\/} of the same flavour as the classical PAC-Bayesian excess risk bound (\ref{eq:firsterboundtogebound}), and with the same exponent $\gamma$. In particular, if ERM is used then the excess risk term will be nonpositive and only the faster-rate term remains. We also provide a class of exchangeable priors for which a Gibbs posterior can be calculated based on the observed data, and for the corresponding Gibbs predictor we also get a bound in which the excess risk term can be omitted (Example~\ref{ex:gibbs}).

Note that the empirical excess risk term in excess risk bounds does not necessarily vanish if $\beta \downarrow 0$: the RHS of our result (\ref{eq:mainpre})  provides the best of the RHS of (\ref{eq:firsterboundtogebound}) and (\ref{eq:firstgebound}). For ERM, if the best $\beta$ in the Bernstein condition is known (e.g., for bounded squared or logistic loss), the bound (\ref{eq:mainpre}) is empirical---it can be calculated from the data only. If, as in classification, we do not know the best $\beta$ in advance, or we do not use ERM so that the $R$ term is hard to quantify without knowing $f^*$, the bound as such cannot be calculated based on the data only; we return to this issue in Section~\ref{sec:conclusion}. 

We may view both the algorithm $A$ and the data-dependent prior $\pi$ as {\em conditional\/} distributions over $\cF$, given the training sample $\sZ_{\bf 0}$, and the vector of unordered pairs $\langle \sZ_{\bf 0}, \sZ_{\bf 1} \rangle$, respectively. Of course, when designing the prior $\pi$ we can also take into account the algorithm $A$: given $\langle \sZ_{\bf 0}, \sZ_{\bf 1} \rangle$, there are only $2^n$ outputs possible for any deterministic algorithm $A$ (such as ERM) that outputs a single distribution given training sample $\sZ_{\bf 0}$. As an additional benefit, we can thus take, without loss of generality, a prior $\pi$ with discrete support of at most $2^n$ elements, allowing us to provide bounds for general nonrandomized learning algorithms -- something which, as we have already seen, is not possible in the standard PAC-Bayesian setup when the parameters of $\cF$ are continuous-valued. 

\paragraph{Solving both Issues for Mutual Information}

As mentioned above, our bound (\ref{eq:mainpre}) holds as an {\em exponential stochastic inequality\/} (Definition~\ref{def:esi}). Formally, an ESI has the following implications.
\begin{proposition}[{ESI Implications \citep[Prop.9]{MhammediGG19}}]\label{prop:esiimpl} 
		If $X \stochleq_{\eta} Y$, then $\forall \delta \in(0,1)$, $X\leq Y + \frac{\log \frac{1}{\delta}}{ \eta}$, with probability at least $1-\delta$. Now let $\bar\eta > 0$ and let $g:[0,\bar\eta]$ be continuous and nondecreasing. If for all $\eta$ with $0 < \eta \leq \bar\eta$, $X \stochleq_{\eta} Y + g(\eta)$, then $\ex{}{X} \leq \ex{}{Y} +g(0).$
\end{proposition}
Our main Theorem~\ref{th:main}, rendered as (\ref{eq:mainpre}) above, holds with $\stochleq$ instantiated to $\stochleq_{\eta}$ with every $0 < \eta \leq c\sqrt{n}$ for some constant $c> 0$. It can thus be weakened, by applying the proposition above with $g(\eta) = 6 \eta / n$, to an in-probability PAC statement (setting $\eta = c \sqrt{n}$, it holds with probability at least $1-\delta$ up to $6c /\sqrt{n} + (- \log \delta)/(c \sqrt{n}) = O(1/\sqrt{n})$) but also to an in-expectation statement in which the remainder term $6 \eta/n$ disappears. We then get a {\em MAC}-Bayesian bound, with MAC standing for `Mean Approximately Correct'. 
By plugging into (\ref{eq:mainpre}) a special almost exchangeable prior that is both distribution- and  data-dependent, namely the prior that minimizes the bound in expectation for the given learning algorithm, we get the corresponding {\em faster-rate conditional mutual information bound\/}: 
\begin{align}\label{eq:firstcmibound}
\ex{\sZ_{\bf 0}}{L(A | \sZ_{\bf 0}; \cD) - L(A | \sZ_{\bf 0}; \sZ_{\bf 0})}
\leq  
 \minim{1}{2\beta} \cdot \ex{\sZ_{\bf 0}}{R(A | \sZ_{\bf 0}; \sZ_{\bf 0})} + O\left( 
 \frac{\CMI{A}{\cD}}{n}\right)^{\frac{1}{2-\beta}}
\end{align}
The term $\CMI{A}{\cD} = \inf_\pi \ex{\sZ_{\bf 0},\sZ_{\bf 1}}{\dkl{A|\sZ_{\bf 0}}{\pi|\langle \sZ_{\bf 0}, \sZ_{\bf 1}\rangle}}$ denotes the \emph{conditional mutual information of} $A$ with respect to data distribution $\cD$, introduced by~\citet{SteinkeZ20} as an information complexity measure, which is always finite, avoiding the impossibility results of~\citet{BassilyMNSY18}. This conditioning approach has already proven useful in proving sharper generalization bounds~\citep{HaghifamNKRD20}. 
However, until the present work, no \emph{fast rate} results had been proven with respect to $\mathsf{CMI}$.

In contrast to the standard bound (\ref{eq:firstmibound}), there are no absolute signs on the left, but this is not of great concern since we are almost always interested in a one-sided bound anyway. If $\beta=0$, the right-hand side of the bound~\eqref{eq:firstcmibound} is smaller than that of~\eqref{eq:firstmibound}, since $\CMI{A}{\cD}\leq I(A|Z;Z)$~\citep{HaghifamNKRD20}. Under a Bernstein condition or bounded loss with curvature, where $\beta>0$, the rate is clearly faster than the rate obtained by the standard CMI bound, albeit with an additional excess risk term. For ERM, this first term disappears, and more generally in most interesting settings, the complexity term is the dominant term.

\paragraph{In Expectation vs.~In Probability -- A Paradox?}
At first sight, a fast rate means `with high probability, convergence happens at rate faster than $O(1/\sqrt{n})$'. But this is impossible even in trivial cases with ${\cal F} = \{f \}$ containing only one element (so every learning algorithm must output $f$, no matter what data are observed -- there is no learning/overfitting): if $\ell(f,Z_1)$ has variance $\sigma^2$, then we find by the central limit theorem that for every fixed $\alpha < 1$, for all large $n$, 
$$
L(A| Z ; \cD) - L(A | Z ;  Z ) = \ex{f\sim F}{\ell(f;\cD)} - \ex{f\sim F}{\ell(f;Z)}
\geq C_{\alpha} \frac{\sigma}{\sqrt{n}}
$$
with probability $\alpha$ over the training sample $Z$ and a constant $C_{\alpha} > 0$.
Yet, (\ref{eq:mainpre}) still provides faster rates in a weaker sense. To see this, note first that, being an ESI, it implies convergence in expectation; and then the $1/\sqrt{n}$ term is really not there (and the Central Limit Theorem does not hurt us)  -- so we do get a faster rate in expectation. Second, the largest subscript $\eta$ for which the ESI holds is of order $\sqrt{n}$ -- implying that we do incur $O(1/\sqrt{n})$-fluctuations, and do not 
contradict the central limit theorem. Yet importantly, the square-root term has been {\em decoupled\/} from the $\kl$ complexity term, which (if $\beta =1$) can converge to $0$ as fast as $O(\kl/n)$. In contrast, all other PAC-Bayes bounds we know of, except those of \citet{MhammediGG19}, have the $\kl/n$ term {\em inside} the square root. If the $\kl$  term grows with $n$, as it usually does, this may make the convergence rate of such classical bounds  substantially  worse than $O(1/\sqrt{n})$. 
Thus, borrowing the terminology of \citet{MhammediGG19}, we really have {\em faster rates in probability up to an {\em irreducible, complexity-free\/} $O(1/\sqrt{n})$ term}.\footnote{For ESI-excess risk bounds, because of the substraction of $\ell(f^*; Z)$ in the bounds, the variance of the excess risk $L(A| Z ; \cD)$ goes to $0$ under a Bernstein condition and fast rates without the $O(1/\sqrt{n})$ term are possible --- indeed, if $\beta > 0$ then (\ref{eq:firsterbound}) holds for an $\eta$ that goes to $0$ slower than $1/\sqrt{n}$ (\citet{GrunwaldM20} provide various examples) and one gets in-probability excess risk bounds without the irreducible $O(1/\sqrt{n})$ term.
}

\subsection{Related Work; Other Extensions of the Standard PAC-Bayesian Equation}\label{sec:related}
Although they sometimes look different, most PAC-Bayes bounds can, potentially after slight relaxation, be brought in the form (\ref{eq:firstgebound}). Examples include  the well-known bound with $\kl$ on the left due to~\citet{langford2003pac,Seeger02,maurer2004} and the standard bound due to \citet{Catoni07}; see also \citet{tolstikhin2013pac}, who provide an overview and discussion of this type of bound.  Based on an empirical Bernstein analysis, \citet{tolstikhin2013pac} replaced the empirical error term inside the square root in (\ref{eq:firstgebound}) by a smaller second order term which, however, still is close to $0$ only when the empirical error itself is close to $0$.
Based on a variation of the empirical Bernstein idea, a lemma which they called  {\em un-expected\/} Bernstein, \citet{MhammediGG19}  replace the empirical error term inside the square root by a different second-order term which, they show, goes to $0$ with high probability whenever a Bernstein condition holds. Thus, they are presumably one of the first to have a fast rate PAC-Bayesian {\em generalization\/} error bound (note again that fast PAC-Bayes {\em excess\/} risk bounds have been known for a long time). Their Theorem 7 provides a first version of the in-probability version of our (\ref{eq:mainpre}), but with the $\minim{1}{2 \beta}$ replaced by $1$ and the empirical excess risk  ${R(A | \sZ_{\bf 0} ;  \sZ_{\bf 0})}$ replaced by (essentially) three times the standard risk (i.e., expected loss difference), making their first term  larger than ours and not converge to $0$ for algorithms for which the excess risk does not converge to $0$; also their analysis is based on  priors that do not allow conditioning on a ghost sample. However, in contrast to our bound, their bound has the pleasant property of being fully empirical, a point to which we return in Section~\ref{sec:conclusion}. Simultaneously, \citet{YangSR19} also gave a fast rate PAC-Bayes generalization bound using a different technique, which includes a so-called `flatness' term attempting to capture the flatness of the empirical risk surface on which the
posterior Gibbs classifier concentrates. If this term is small with high probability, then the bound converges fast. In contrast, our bound converges fast when the strong Bernstein condition ($\gamma=1$) holds and achieves faster rates otherwise. It is easy to show the `flatness' term of~\citep{YangSR19} can be large even if a strong Bernstein condition holds; on the other hand, there may also be cases in which their bound is tighter than ours --- the bounds are so different that they are hard to compare in general. 

\paragraph{The Other Type of Data-Dependent Prior}
\citet{MhammediGG19} do make use of data-dependent priors, an idea pioneered by \citet{ambroladze2007tighter}, which is to set aside part of the training data and condition everything on it.  In the simplest instance, one uses the learning algorithm's output on the first half of the data as a prior, then performs a standard  PAC-Bayesian bound such as (\ref{eq:firstgebound}) on the second half. In this way one looses a factor of 2 in the bound but gets a much better informed prior, making the final bound often substantially better in practice (e.g. in~\citep{DziugaiteHGAR21}). 
\citet{MhammediGG19} extend this idea to using both half samples and mixing the results, analogously to cross-validation. Note though that this is a very different type of data-dependency than ours: the prior is given the full first half of the sample, rather than the full training sample plus a ghost sample with ordering information removed. 

\paragraph{The Core of Our Contribution}
{\em MAC}-Bayesian bounds, although not under that name, are already to be found in  Catoni's monograph \citep{Catoni07}. Catoni already mentions that the prior that minimizes a MAC-Bayesian bound is the prior that turns the $\kl$ term into the mutual information.  Moreover, \citet{Catoni07}, as well as~\citet{Audibert04} in his Ph.D. thesis, introduce the expected $\kl$ complexity based on almost exchangeable priors conditioned on a supersample, but these are not connected to conditional mutual information as in our paper. Even more closely related,~\citet{HellstromD20} introduced an exponential inequality which yields conditional PAC-Bayesian and in-expectation bounds. However, none of the previous works connects fast rates to the conditional case with almost exchangeable prior. This is the crucial contribution of the present paper, hinging on our main, and novel, technical Lemma~\ref{lem:esiindividual}, which allows us to get fast rates. Below the lemma we explain how it goes beyond earlier developments. 

\section{Preliminaries}
\label{sec:preliminaries}
\paragraph{Additional Notation}
For convenience, we include a glossary with all frequently used symbols in Appendix~\ref{app:glossary}. 
For a random variable $X$ and a distribution $\cP$, we write $X\sim\cP$ to denote that $X$ is drawn from $\cP$ and $X\sim\cP^n$ to denote that $X$ consists of $n$ i.i.d.\ draws from $\cP$. The distribution of a random variable $X$ is denoted by $\cP_X$ and will be omitted when it is clear from context. We denote the Bernoulli distribution over $\{0,1\}$ with mean $p$ by $\Ber(p)$. We also write $[n]=\{1,\ldots, n\}$.

A supersample $\sZ=((\sZ_{1,0}, \sZ_{1,1}), \ldots, (\sZ_{n,0}, \sZ_{n,1}))^\top\sim \cD^{n\times 2}$ is a vector of $n$ pairs of i.i.d.\ examples, as in Table~\ref{tab:supersample}. 
Let $S = (S_1, \ldots, S_n) \in\{0,1\}^n$ such that $S\sim\Ber(1/2)^n$ and let $\bar{S}_i=1-S_i$ for all $i\in[n]$. 
We write $\sZ_S=(\sZ_{1,S_1}, \ldots, \sZ_{n,S_n})^\top\in\cZ^{n\times 1}$ for the sub-vector of $\sZ$ indexed by $S$ and $\sZ_{\bar{S}}=(\sZ_{1,\bar{S}_1}, \ldots, \sZ_{n,\bar{S}_n})^\top$ for its complement. 
Note that with this notation, we can write $\sZ=(\sZ_{\mathbf{0}},\sZ_{\mathbf{1}})$, setting $S=\mathbf{0}$.
We also refer to the vector of \emph{unordered} pairs $\langle \sZ_{\bf 0},\sZ_{\bf 1}\rangle= (\{\sZ_{1,0}, \sZ_{1,1}\},\ldots,\{\sZ_{n,0},\sZ_{n,1}\})^\top$. With this notation, for any \emph{almost exchangeable} prior distribution $\pi: \cZ^{n \times 2} \rightarrow \Delta(\cF)$ (Definition~\ref{def:exprior}) it holds that for all $\sz\in\cZ^{n\times 2}$, $\forall s\in\{0,1\}^n$, $\pi|\sz=\pi|(\sz_s,\sz_{\bar{s}})=\pi|\langle \sz_s,\sz_{\bar{s}}\rangle =\pi|\langle \sz_{\bf 0}, \sz_{\bf 1} \rangle$.

    \begin{table}[H]
        \centering
        \begin{tabular}{|c|c|}
        \hline
            $\sZ_{1,0}$  & $\sZ_{1,1}$ \\
            $\sZ_{2,0}$ & $\sZ_{2,1}$ \\
            $\vdots$ & $\vdots$ \\
            $\sZ_{n,0}$ & $\sZ_{n,1}$\\
            \hline
        \end{tabular}
        \caption{Supersample $\sZ\in\cZ^{n\times 2}$}
        \label{tab:supersample}
    \end{table}

\subsection{KL divergence and Mutual Information}\label{app:prelim-it}
First, we define the $\kl$ divergence of two distributions.
\begin{definition}[KL Divergence]
 Let $\cP, \cQ$ be two distributions over the space $\Omega$ and suppose $\cP$ is absolutely continuous with respect to $\cQ$. The \emph{Kullback–Leibler ($\kl$) divergence} (or \emph{relative entropy}) from $\cQ$ to $\cP$ is \[\dkl{\cP}{\cQ}=\ex{X\sim \cP}{\log\frac{\cP(X)}{\cQ(X)}},\]
 where $\cP(X)$ and $\cQ(X)$ denote the probability mass/density functions of $\cP$ and $\cQ$ on $X$, respectively.\footnote{Formally, $\frac{\cP(X)}{\cQ(X)}$ is the Radon-Nikodym derivative of $\cP$ with respect to $\mathcal{Q}$. If $P$ is not absolutely continuous with respect to $\cQ$ (i.e., $\frac{\cP(X)}{\cQ(X)}$ is undefined or infinite), then the $\kl$ divergence is defined to be infinite.}
\end{definition}

Next, we define  mutual information.
\begin{definition}[Mutual Information]
 Let $X,Y$ be two random variables jointly distributed according to $\cP$. 
 The mutual information of $X$ and $Y$ is
$$I(X;Y)=\dkl{\cP_{(X,Y)}}{\cP_X\times\cP_Y}= \ex{X}{\dkl{\cP_{Y|X}}{\cP_Y}},$$
 where by $\cP_X\times\cP_Y$ we denote the product of the marginal distributions of $\cP$ and $\cP_{Y|X=x}(y)=\cP_{(X,Y)}(x,y)/\cP_X(x)$ is the conditional density function of $Y$ given $X$.
\end{definition}

\begin{definition}[Conditional Mutual Information]
For random variables $X,Y,Z$, the mutual information of $X$ and $Y$ conditioned on $Z$ is 
$$I(X;Y \mid Z)=I(X;(Y,Z))-I(X;Z).$$
\end{definition}

We define here the less common notion of \emph{disintegrated mutual information},  as in~\citep{NegreaHDKR19, HaghifamNKRD20}.
\begin{definition}[Disintegrated Mutual Information]
The \emph{disintegrated mutual information} between $X$ and $Y$ given $Z$ is
$$I^Z(X;Y)=\dkl{\cP_{(X,Y)|Z}}{\cP_{X|Z}\times\cP_{Y|Z}},$$
where $\cP_{(X,Y)|Z}$ denotes the conditional joint distribution of $(X,Y)$ given $Z$ and $\cP_{X|Z},\cP_{Y|Z}$ denote the conditional marginal distributions of $X$, $Y$ given $Z$, respectively.

The expected value of this quantity over $Z$ is the Conditional Mutual Information between $X$ and $Y$ given $Z$ that was defined above: $I(X;Y|Z)=\exinline{Z}{I^Z(X;Y)}$.
\end{definition}

We now define the Conditional Mutual Information of an Algorithm, as introduced in~\citep{SteinkeZ20}.

\begin{definition}[{Conditional Mutual Information (CMI) of an Algorithm \citep{SteinkeZ20}}]\label{def:CMI}
Let $A:\cZ^n\rightarrow \Delta(\cF)$ be a randomized or deterministic algorithm.
Let $\cD$ be a probability distribution on $\cZ$ and let $\sZ\in\cZ^{n \times 2}$ be a supersample consisting of $n$ pairs of examples, each example drawn independently from $\cD$. Let $S\sim\Ber(1/2)^n$, independent from $\sZ$ and the randomness of $A$. Let $\sZ_S = (\sZ_{1,S_1},\ldots, \sZ_{n,S_n})^\top\in\cZ^n$ -- that is, $\sZ_S$ is the subset of $\sZ$ indexed by $S$. 

The \emph{conditional mutual information (CMI) of $A$ with respect to $\cD$} is $$\CMI{A}{\cD}:=I(A|\sZ_S;S \mid \sZ)=\ex{\sZ}{I^{\sZ}(A|\sZ_S;S)}.$$

\end{definition}

\subsection{ESI Calculus}\label{app:prelim-esi}
The following proposition is useful for our proofs.
\begin{proposition}[{ESI Transitivity and Chain Rule \citep[Prop.10]{MhammediGG19}}]\label{prop:esichainrule}
\begin{enumerate}[label=(\alph*)]
\item Let $Z_1, \ldots, Z_n$ be any random variables in $\cZ$ (not necessarily independent). If for some $(\gamma_i)_{i\in[n]}\in(0,+\infty)^n$, $Z_i \stochleq_{\gamma_i} 0$ for all $i\in[n]$, then
\[\sum_{i=1}^n Z_i \stochleq_{v_n} 0, \text{ where $v_n = \left(\sum_{i=1}^n\frac{1}{\gamma_i}\right)^{-1}$}.\]
\item Suppose now that $Z_1, \ldots, Z_n$ are independent and for some $\eta > 0$, for all $i \in [n]$, we have $Z_i \stochleq_{\eta} 0$. Then $\sum_{i=1}^n Z_i \stochleq_{\eta} 0$.
\end{enumerate}
\end{proposition}

We now state a basic PAC-Bayesian result we use, under the ESI notation: 
\begin{proposition}[{ESI PAC-Bayes \citep[Prop.11]{MhammediGG19}}]\label{prop:esiPACBayes}
Fix $\eta > 0$ and let $\{Y_f: f \in \cF\}$ be any family of random variables such that for all $f \in \cF$, $Y_f \stochleq_{\eta} 0$. Let $\pi\in\Delta(\cF)$ be any distribution on $\cF$ and let $A: \bigcup_{i=1}^n \cZ^{i} \rightarrow \Delta(\cF)$ be a possibly randomized learning algorithm. Then 
\begin{align*} \ex{f \sim A|Z}{Y_{f}} \stochleq_{\eta}^{Z} \frac{\dkl{A|Z }{\pi}}{\eta}.  \end{align*} 
\end{proposition}

Inside the proof of our main result we work with a random (\emph{i.e.}, data-dependent) $\hat\eta$ in the ESI inequalities. We extend Definition \ref{def:esi} to this case by replacing the expectation in the definition of ESI  by the expectation over the joint distribution of ($X$, $Y$, $\hat\eta$): $X \stochleq_{\hat\eta} Y$ means that  $\ex{}{\exp(\hat\eta (X-Y))} \leq 0$.
Via the following proposition one can tune $\eta$ after seeing the data. 
	\begin{proposition}[{ESI from fixed to random $\eta$ \citep[implied by Prop.12]{MhammediGG19}}]
		\label{prop:randeta}
		Let $\cG$ be a countable subset of ${\mathbb R}^+$ such that, for some $\eta_0 > 0$, for all $\eta \in \cG$, $\eta \geq \eta_0$. Let $\pi$ be a probability mass function over $\cG$. Given a countable collection $\{Y_{\eta} : \eta \in \cG \}$ of random variables satisfying $Y_{\eta} \stochleq_{\eta} 0$, for all fixed $\eta \in \cG$, we have, for arbitrary estimator $\estimate{\eta}$ with support on $\cG$,
		\begin{align*}
		Y_{\estimate\eta} \stochleq_{\eta_0} \frac{- \log \pi(\estimate\eta)}{\estimate\eta}.
		\end{align*}
	\end{proposition}
	
\subsection{Bernstein Condition}\label{sec:bernstein}
We consider learning problems $(\cD,\ell,\cF)$ which satisfy the Bernstein Condition (Definition~\ref{def:Bernstein} in Section~\ref{sec:intro}). It will be convenient to work with the following {\em linearized version\/} of the Bernstein condition, proven in Appendix~\ref{app:proofs}. It extends a well-known result that has appeared in previous work, e.g. in~\citep{WooterGE16}.
\begin{proposition}\label{prop:linearizedBernstein}
Suppose that $(\cD, \ell, \cF)$ satisfies the {$(B,\beta^*)$-Bernstein condition} for $\beta^* \in [0,1]$. Pick any $c > 0, \eta < 1/(2 B c)$. Then for all $0 < \beta \leq \beta^*$ and for all $f\in\cF$:
$$c \cdot \eta \ex{Z'\sim\cD}{(\ell(f;Z')-\ell(f^*; Z'))^2} \leq \minim{\frac{1}{2}}{\beta} \cdot \left(\ex{Z'\sim\cD}{\ell(f;Z')-\ell(f^*; Z')}\right)  ~ 
+ (1- \beta)   \cdot (2 Bc \eta)^{\frac{1}{1-\beta}} 
$$ 
\end{proposition}
Note that, by our assumption on $\eta$, $\lim_{\beta \uparrow 1} (2Bc\eta)^{1/(1-\beta)} = 0$ and the second term vanishes for $\beta=1$.

\section{Main Development}
\begin{lemma}[Main technical lemma]\label{lem:esiindividual}
Fix any two real numbers $r_0,r_1$ such that $|r_0|, |r_1| \leq 1$. Let $S\sim\Ber(1/2)$ and let $\bar{S} = 1- S$. Then for all $0 < \eta \leq 1/4$ it holds that
$$
r_{\bar{S}} - r_{S} \stochleq_{\eta}  \eta \cdot C_{\eta} r^2_{\bar{S}} \leq \eta \cdot C_{1/4} r^2_{\bar{S}} 
$$
where $C_0 = 2$, $C_{\eta}$ is a continuous increasing function of $\eta$ and $C_{1/4} \approx3.6064$.
\end{lemma}
The proof of this bound, with an explicit formula for the constant $C_{\eta}$, is in Appendix~\ref{app:proofs}. Our formula for $C_\eta$ is tight near $\eta=0$ but can be improved if it is known that $r_0,r_1$ are of the same sign. For ease of exposition, below we will only use the value for $\eta=1/4$.  

The lemma is the cornerstone in the proof of our main theorem which now follows. In this proof, $r_S$ is set to the excess loss of a hypothesis $f$ on an example from sample $\sZ_S$. Crucially, the square term on the right, when applied in the proof, only refers to a ghost sample $\sZ_{\bar S}$ while $f$ is a hypothesis trained on the real sample $\sZ_{\bf 0}$ -- this allows us to `kill' it under a Bernstein condition, replacing the square by a small enough linear term. 
A qualitatively similar inequality which has the sum $r_{\bar{S}}^2 + r_{S}^2$ 
on the right implicitly appears in \citep{Audibert04}, but these square terms, being a combination of training and ghost samples, are not easily removed in our proof, and to get a PAC-Bayesian bound based on this lemma one needs to pick $\eta$ small enough so that the term becomes negligible, leading to $\eta \asymp 1/\sqrt{n}$ which implies slow rates. Killing the square terms by taking a very small $\eta$ also happens implicitly in the proof of the $\mathsf{CMI}$ result of \citet{SteinkeZ20} as well as~\citet{HellstromD20} which, for this reason, also give the slow rate. 

We note that our Lemma~\ref{lem:esiindividual} does not hold for unbounded losses and specifically does not hold for sub-Gaussian losses (to see this, for example, consider the case of $r_0 = 0$ and $r_1 < (- \ln 2)/\eta$). Adjusting this lemma for sub-Gaussian losses yields terms on the right-hand side that only lead to slow rates -- a similar issue as the one described above occurring in prior work~\citep{SteinkeZ20, HellstromD20, Audibert04}. Thus, while a similar result as ours {\em might\/} hold for sub-Gaussian losses, it would require fundamentally new ideas to prove it.

\begin{theorem}\label{th:main}
Let $(\cD,\ell,\cF)$ represent a learning problem which satisfies the $(B,\beta^*)$-Bernstein condition and suppose the loss function $\ell$ has range in $[0,1]$. Let $A:\bigcup_{i=1}^n \cZ^i \rightarrow \Delta(\cF)$ be a possibly randomized learning algorithm and $\pi\in\Delta(\cF)$ be any almost exchangeable prior. Let $\sZ_{\bf 0}, \sZ_{\bf 1}$ be two samples of $n$ i.i.d.\ examples each drawn from $\cD$. Then, for all $\beta\in[0,\beta^*]$, all $0 < \eta \leq \sqrt{n} \etamax/24 $,  it holds that,
\begin{multline}\label{eq:main}
{L(A| \sZ_{\bf 0} ; \cD) - L(A | \sZ_{\bf 0} ;  \sZ_{\bf 0})} 
\stochleq^{\sZ_{\bf 0}}_{\eta}\\ \minim{1}{2\beta} \cdot {R(A | \sZ_{\bf 0} ;  \sZ_{\bf 0})} +  8 \cdot  \left(\frac{\ex{\sZ_{\bf 1}}{\dkl{A | \sZ_{\bf 0}}{\pi|  \langle \sZ_{\bf 0}, \sZ_{\bf 1} \rangle}}+ \llog n }{n\etamax}  \right)^{\frac{1}{2-\beta}}_{[**]} + \frac{6 \eta}{n},
\end{multline}
where $\etamax= \minim{\frac{1}{4}}{\frac{1}{2BC_{1/4}}}$, $C_{1/4} =3.6064$, $\llog n = \log (\lceil \log_2(\sqrt{n}) \rceil +2) = O(\log\log n)$ and the notation $a^b_{[**]}$ stands for $\max\{a^b,a\}$. 
\end{theorem}
In all interesting cases, the quantity $a$ inside the notation $a^b_{[**]}$ in the bound 
is less than $1$, thus $a^b_{[**]}=a^b$.
Otherwise, the bound would not be useful, as the LHS is less than $1$ for any loss in $[0,1]$. 

Since this is still a $\stochleq_{\eta}$-ESI statement with $\eta= \sqrt{n}\etamax/24$, it implies the in-probability statement that with probability at least $1-\delta$, the above holds up to an additional $(- \log \delta)/\eta$ term on the right. More formally, a simple application of Proposition~\ref{prop:esiimpl} to ESI~\eqref{eq:main} of Theorem~\ref{th:main} yields Corollary~\ref{cor:inprob}:

\begin{corollary}\label{cor:inprob}
Consider the setting and notation of Theorem~\ref{th:main}. Let $\delta\in(0,1)$. For all $\beta\in[0,\beta^*]$ and all almost exchangeable priors $\pi$, with probability $1-\delta$ over the choice of $\sZ_{\bf 0}\sim\cD^n$, we have
\begin{multline*}
L(A| \sZ_{\bf 0} ; \cD) - L(A | \sZ_{\bf 0} ;  \sZ_{\bf 0}) 
\leq
\minim{1}{2\beta} \cdot R(A | \sZ_{\bf 0} ;  \sZ_{\bf 0}) \\ +  8 \cdot  \left(\frac{\ex{\sZ_{\bf 1}}{\dkl{A | \sZ_{\bf 0}}{\pi|  \langle \sZ_{\bf 0}, \sZ_{\bf 1} \rangle}}+ \llog n }{n\etamax}  \right)^{\frac{1}{2-\beta}}_{[**]} + \frac{\etamax}{4\sqrt{n}} +\frac{24\log(1/\delta)}{\sqrt{n}\etamax} .
\end{multline*}
\end{corollary}

The bound~\eqref{eq:main} also implies the corresponding in-expectation statement with the remainder term $6\eta/n$ set to $0$. However, if one directly sets out to prove it, the term $\llog n$ and a factor of $2$ from the multiplicative constant in front of the $a^b_{[**]}$ term can be avoided. In particular, the following improved bound holds, whose proof is based on the proof of Theorem~\ref{th:main} and is in Appendix~\ref{app:proofs}.
\begin{corollary}\label{cor:inexpect}{\rm (`{\bf Variation} of Theorem~\ref{th:main}')}
Consider the setting and notation of Theorem~\ref{th:main}. For all $\beta\in[0,\beta^*]$, it holds that
\begin{multline}\label{eq:finalinexpect}
 \ex{\sZ_{\bf 0}}{L(A | \sZ_{\bf 0}; \cD) - L(A | \sZ_{\bf 0}; \sZ_{\bf 0})}
 \leq \\
 \minim{1}{2\beta} \cdot \ex{\sZ_{\bf 0}}{R(A | \sZ_{\bf 0}; \sZ_{\bf 0})} + 4 \cdot \left( 
 \frac{\ex{\sZ_{\bf 0}, \sZ_{\bf 1}}{\dkl{A | \sZ_{\bf 0}}{\pi|  \langle \sZ_{\bf 0}, \sZ_{\bf 1} \rangle}}}{n\etamax}\right)^{\frac{1}{2-\beta}}_{[**]}.
\end{multline}
\end{corollary}

Moreover, for the right choice of prior, the expected $\kl$ term is $\CMI{A}{\cD}$, implying the bound:
\begin{corollary}\label{cor:cmi}
Consider the setting and notation of Theorem~\ref{th:main}. For all $\beta\in[0,\beta^*]$, it holds that
\begin{equation*}
 \ex{\sZ_{\bf 0}}{L(A | \sZ_{\bf 0}; \cD) - L(A | \sZ_{\bf 0}; \sZ_{\bf 0})}
 \leq 
 \minim{1}{2\beta} \cdot \ex{\sZ_{\bf 0}}{R(A | \sZ_{\bf 0}; \sZ_{\bf 0})} + 4 \cdot \left( 
 \frac{\CMI{A}{\cD}}{n\etamax}\right)^{\frac{1}{2-\beta}}_{[**]}.
\end{equation*}
\end{corollary}
\begin{proof}[Proof of Corollary~\ref{cor:cmi}]
Let $\sZ=(\sZ_{\bf 0}, \sZ_{\bf 1})$. We focus on the $\kl$ divergence in the bound~\eqref{eq:finalinexpect}:
$$
\ex{\sZ_{\bf 0}, \sZ_{\bf 1}}{\dkl{A | \sZ_{\bf 0}}{\pi|  \langle \sZ_{\bf 0}, \sZ_{\bf 1} \rangle}} 
=  \ex{S, \sZ_{\bf 0}, \sZ_{\bf 1}}{\dkl{A | \sZ_{\bf 0}}{\pi|  \langle \sZ_{\bf 0}, \sZ_{\bf 1} \rangle}}
= \ex{S, \sZ}{{\dkl{A | \sZ_{S}}{\pi|  \langle \sZ_{S}, \sZ_{\bar{S}} \rangle}}}
$$
The first equality holds since $S$ is independent of $\sZ_{\bf 0}, \sZ_{\bf 1}$. The second equality holds because the distributions of $\sZ_{S}$, $\sZ_{\bar{S}}, \sZ_{\bf 0}, \sZ_{\bf 1}$ are all identical to $\cD^n$ and $\pi$ is almost exchangeable. 
We choose $\pi=\ex{S'}{A | \sZ_{S'}}$ for $S' \sim\Ber(1/2)^n$. Notice that $\pi$ is indeed almost exchangeable. We now have
$$
\ex{S, \sZ}{{\dkl{A | \sZ_{S}}{\ex{S'}{A|\sZ_{S'}}}}}
=\ex{\sZ}{\ex{S}{\dkl{A | \sZ_{S}}{\ex{S'}{A|\sZ_{S'}}}}}
=\ex{\sZ}{I^{\sZ}(A|\sZ_S;S)}
=\CMI{A}{\cD}.
$$
Combining the two equations and substituting the term in inequality~\eqref{eq:finalinexpect} completes the proof.
\end{proof}

After observing the implications of Theorem~\ref{th:main}, we now present its complete proof below.
\begin{proof}[Proof of Theorem~\ref{th:main}]
Let $\sz = ( (\sz_{1,0}, \sz_{1,1}), \ldots, (\sz_{n,0}, \sz_{n,1}))^\top
\in\cZ^{n\times 2}$ be a given, fixed supersample. Let $S= (S_1, \ldots, S_n)$, with $S_1, S_2, \ldots, S_n$ i.i.d.\ $\Ber(1/2)$, be a selection vector and let $\bar{S}$ be its complement, that is, $\bar{S}_i := 1- S_i$ for all $i\in[n]$. 
For each fixed $f\in\cF$ and $\sz\in\cZ^{n\times 2}$, we define
$$r_i(f;\sz_{i,0}) = \ell(f;\sz_{i,0}) - \ell(f^*;\sz_{i,0})
\ \text{ and } \ r_i(f;\sz_{i,1}) = \ell(f;\sz_{i,1}) - \ell(f^*;\sz_{i,1}).
$$

Since $\ell$ has range in $[0,1]$, it holds that for all $i\in[n]$, $|r_i(f;\sz_{i,0})|, |r_i(f;\sz_{i,1})|\leq 1$. 
By Lemma~\ref{lem:esiindividual}, for all $i\in[n]$, and $\eta<1/4$, it holds that 
\begin{equation}\label{eq:esiindividual}
r_i(f;\sz_{i,\bar{S}_i})-r_i(f;\sz_{i,S_i}) \stochleq_{\eta}^{S_i} \eta C_{1/4} r^2_i(f,\sz_{i,\bar{S}_i})
\end{equation}
Now take randomness under the product distribution $\Ber(1/2)^n$ of $S$. By independence of the $S_i$ and applying Proposition~\ref{prop:esichainrule}, we can then add the $n$ ESIs~\eqref{eq:esiindividual} to give: 
\begin{align*}
 \sum_{i=1}^n r_i(f;\sz_{i,\bar{S}_i}) - \sum_{i=1}^n r_i(f;\sz_{i,S_i}) & \stochleq_{\eta}^S
 \eta C_{1/4} \sum_{i=1}^n r^2_i(f,\sz_{i,\bar{S}_i}).
 \end{align*}
 
Now consider a learning algorithm $A$ that outputs a distribution $A| \sz_S $ on $\cF$, and any `prior' distribution $\pi|\sz$ on $\cF$ that is allowed to depend on $\sz$ (which for now is considered fixed). The PAC-Bayes theorem (Proposition~\ref{prop:esiPACBayes}) gives
\begin{align}\label{eq:esisum}
\ex{f \sim A | \sz_S}{\sum_{i=1}^n r_i(f;\sz_{i,\bar{S}_i}) - \sum_{i=1}^n r_i(f;\sz_{i,S_i})} & \stochleq^{S | \sz}_{\eta}
\eta C_{1/4} \ex{f \sim A | \sz_S} { \sum_{i=1}^n r^2_i(f,\sz_{i,\bar{S}_i})} + \frac{\dkl{A | \sz_S}{\pi|  \sz}}{\eta}.
\end{align}

We note that $S$ is independent of $\sz$, so the ESI above could be equivalently written with respect to $S$ instead of $S|\sz$.

Since inequality~\eqref{eq:esisum} holds for {\em all\/} $\sz$, we weaken it to an ESI by taking its expectation over $\sZ\sim\cD^{n\times 2}$:
\begin{align*}
\ex{f \sim A |\sZ_S}{\sum_{i=1}^n r_i(f;\sZ_{i,\bar{S}_i}) - \sum_{i=1}^n r_i(f;\sZ_{i,S_i})} & \stochleq^{S,\sZ}_{\eta} 
\eta C_{1/4} \ex{f \sim A |\sZ_S} { \sum_{i=1}^n r_i^2(f,\sZ_{i,\bar{S}_i})} + \frac{\dkl{A|\sZ_S}{\pi|\sZ}}{\eta}
\end{align*}

Since the conditional distribution $\pi$ is almost exchangeable with respect to $\sz$, the above is rewritten as 
\begin{align*}
\ex{f \sim A |\sZ_S}{\sum_{i=1}^n r_i(f;\sZ_{i,\bar{S}_i}) - \sum_{i=1}^n r_i(f;\sZ_{i,S_i})} & \stochleq^{S,\sZ}_{\eta} 
\eta C_{1/4} \ex{f \sim A |\sZ_S} { \sum_{i=1}^n r_i^2(f,\sZ_{i,\bar{S}_i})} + \frac{\dkl{A|\sZ_S}{\pi|(\sZ_S,\sZ_{\bar{S}})}}{\eta}.
 \end{align*}
Now, since the $\sZ_{1,0}, \sZ_{1,1}, \ldots, \sZ_{n,0}, \sZ_{n,1}$ are i.i.d.\ and independent of the $S_i$, we  must also have: 
\begin{align*}
\ex{f \sim A| \sZ_{\mathbf{0}}}{\sum_{i=1}^n r_i(f;\sZ_{i,1}) - \sum_{i=1}^n r_i(f;\sZ_{i,0})} & \stochleq^{\sZ}_{\eta}
 \eta C_{1/4} \ex{f \sim A| \sZ_{\mathbf{0}}}{ \sum_{i=1}^n r_i^2(f,\sZ_{i,1})} + \frac{\dkl{A | \sZ_{\mathbf{0}}}{\pi|\langle\sZ\rangle}}{\eta},
 \end{align*}
 where we also replaced $\pi|\sz$ by its equivalent $\pi|\langle\sz\rangle$.
Since the $\sZ_{\mathbf{0}},\sZ_{\mathbf{1}}$ consist of i.i.d.\ random variables, we can weaken the above inequality to an in-expectation inequality (by Proposition~\ref{prop:esiimpl}) with respect to the `ghost'' sample $\sZ_{\mathbf{1}}\sim \cD^n$:
\begin{multline}\label{eq:esisumghost}
\ex{f \sim A| \sZ_{\bf 0}}{\ex{\sZ_{\bf 1}}{\sum_{i=1}^n r_i(f;\sZ_{i,1}) - \sum_{i=1}^n r_i(f;\sZ_{i,0})}}  \stochleq^{\sZ_{\bf 0}}_{\eta}\\
 \eta C_{1/4} 
 \ex{f \sim A| \sZ_{\bf 0}}{\ex{\sZ_{\bf 1}}{ \sum_{i=1}^n r_i^2(f,\sZ_{i,1})} }+ 
 \ex{\sZ_{\bf 1}}{\frac{\dkl{A | \sZ_{\bf 0}}{\pi | \langle\sZ\rangle}}{\eta}}.
 \end{multline}

We now focus on term of the expected sum of squared excess risks in the RHS. By applying the linearized $(B,\beta^*)$-Bernstein condition of Proposition~\ref{prop:linearizedBernstein} and adding the inequalities for all $i\in[n]$, we have that for all $\eta< 1/(2BC_{1/4})$, $\beta\in[0,\beta^*]$,
\begin{equation}\label{eq:addlinearizedBern}
\eta C_{1/4} \ex{\sZ_{\bf 1}}{\sum_{i=1}^n r_i^2(f,\sZ_{i,1})} \leq \minim{\frac{1}{2}}{\beta} \cdot\ex{\sZ_{\bf 1}}{\sum_{i=1}^n r_i(f;\sZ_{i,1})} + n (1-\beta)(2BC_{1/4}\eta)^{1/(1-\beta)}.
\end{equation}

Now, observe that $\ex{\sZ_{\bf 1}}{\sum_{i=1}^n r_i(f;\sZ_{i,1})}=n\cdot R(f;\cD)$ 
and $\ex{\sZ_{\bf 1}}{\sum_{i=1}^n r_i(f;\sZ_{i,0})}=n\cdot R(f;\sZ_{\bf 0})$. 
Combining inequality~\eqref{eq:esisumghost} with~\eqref{eq:addlinearizedBern} and substituting the terms above, we have that for all $\eta<  \etamax:= \minim{\frac{1}{4}}{\frac{1}{2BC_{1/4}}}$,
\begin{multline*}
\ex{f \sim A| \sZ_{\bf 0}}{n\cdot R(f;\cD) - n\cdot R(f;\sZ_{\bf 0})}  \stochleq^{\sZ_{\bf 0}}_{\eta}\\
  \minim{\frac{1}{2}}{\beta}\cdot
 \ex{f \sim A| \sZ_{\bf 0}}{n\cdot R(f;\cD)}+ n(1-\beta)(2BC_{1/4}\eta)^{1/(1-\beta)} +
 \ex{\sZ_{\bf 1}}{\frac{\dkl{A | \sZ_{\bf 0}}{\pi | \langle\sZ\rangle}}{\eta}}.
 \end{multline*}

Dividing by $n$ and substituting for the expected true and empirical excess risk of the randomized estimator $A|\sZ_{\bf 0}$, we have the following ESI:
\begin{multline}\label{eq:finalesisum}
R(A | \sZ_{\bf 0}; \cD) - R(A | \sZ_{\bf 0}; \sZ_{\bf 0}) 
\stochleq^{\sZ_{\bf 0}}_{n \eta} 
 \minim{\frac{1}{2}}{\beta} \cdot R(A | \sZ_{\bf 0}; \cD) +  \left( \frac{\eta}{\etamax} \right)^{\frac{1}{1-\beta}} + 
 \frac{\ex{\sZ_{\bf 1}}{\dkl{A | \sZ_{\bf 0}}{\pi|  \langle\sZ\rangle}}}{n \eta}.
 \end{multline}
Using Proposition~\ref{prop:randeta}, we now extend this ESI to deal with random $\eta$. The proposition immediately gives that for every finite  grid $\cG \subset [\etamin,\etamax]$, for arbitrary probability mass function $\pi_\cG$ on $\cG$, for arbitrary functions (random variables) $\hat\eta: \sZ_{\bf 0} \rightarrow \cG$, we have: 
\begin{multline}\label{eq:finalesisumb}
R(A | \sZ_{\bf 0}; \cD) - R(A | \sZ_{\bf 0}; \sZ_{\bf 0}) 
\stochleq^{\sZ_{\bf 0}}_{n \etamin}  
 \minim{\frac{1}{2}}{\beta} \cdot R(A | \sZ_{\bf 0}; \cD) +  \left(\frac{\hat\eta}{\etamax}\right)^{\frac{1}{1-\beta}} + 
 \frac{\textsc{ub} -\log \pi_\cG(\hat\eta)}{n \hat\eta},
 \end{multline}
where $\textsc{ub}$ can be any upper bound on $\ex{\sZ_{\bf 1}}{\dkl{A | \sZ_{\bf 0}}{\pi|  \langle\sZ\rangle}}$. In the remainder of the proof we simply set $\textsc{ub} = \exinline{\sZ_{\bf 1}}{\dkl{A | \sZ_{\bf 0}}{\pi|  \langle\sZ\rangle}}$, the possibility to take a larger upper bound is explored in Example~\ref{ex:gibbs}. 

 Now let  $\pi_\cG$ be the uniform distribution over the grid 
\begin{align} \label{eq:grid} \cG \coloneqq \left\{\etamax, \frac{1}{2}\etamax , \dots, \frac{1}{2^K}\etamax: K \coloneqq \left\lceil \log_2 \left(\sqrt{n}\right) \right\rceil + 2 \right\} \end{align}
and define $\hat\eta'$, as function of data $\sZ_{\bf 0}$ to be the element of $[0,\etamax]$ minimizing the sum 
$$
\textsc{comp}(\eta) = \left(\frac{\eta}{\etamax}\right)^{\frac{1}{1-\beta}} + 
 \frac{\ex{\sZ_{\bf 1}}{\dkl{A | \sZ_{\bf 0}}{\pi|  \sZ}} -\log \pi_\cG(\eta)}{n \eta}
$$ of the last two terms in (\ref{eq:finalesisumb}), and let $\hat\eta$ be the element within $\cG$ that minimizes this sum. We can determine $\hat\eta'$ by differentiation. 
We find that, since we have  $|\cG|=K+1 \geq 3$ and hence $- \log \pi_\cG(\hat\eta) \geq 1$, it holds 
\begin{align*}
    \textsc{comp}(\hat\eta) \leq \begin{cases}
  2 \cdot \textsc{comp}(\hat\eta')=  4 \left(\scalebox{1.1}{$\frac{\ex{\sZ_{\bf 1}}{\dkl{A | \sZ_{\bf 0} }{\pi|  \sZ}}+ \llog n}{n \etamax}$}\right)^{1/(2-\beta)}
& \text{\ if\ } \hat\eta' < \etamax \\
\textsc{comp}(\hat\eta') \leq 2 \left(\scalebox{1.1}{$\frac{\ex{\sZ_{\bf 1}}{\dkl{A | \sZ_{\bf 0} }{\pi|  \sZ}}+ \llog n}{n \etamax}$}\right)
& \text{\ if\ } \hat\eta' = \etamax
    \end{cases}
\end{align*}
where $\llog n = \log (\lceil \log_2 (\sqrt{n}) \rceil + 2) =  O(\log \log n)$.
Combining this with (\ref{eq:finalesisumb}) gives
\begin{multline}\label{eq:finalesisumc}
R(A | \sZ_{\bf 0}; \cD) - R(A | \sZ_{\bf 0}; \sZ_{\bf 0}) 
\stochleq^{\sZ_{\bf 0}}_{n \etamin} \alpha
  \cdot R(A | \sZ_{\bf 0}; \cD) +  
 4  \cdot  \left(\frac{\ex{\sZ_{\bf 1}}{\dkl{A | \sZ_{\bf 0}}{\pi|  \sZ}} + \llog n }{n \etamax}\right)^{1/(2-\beta)}_{[**]}
 \end{multline}
for every $0 < \etamin \leq \frac{\etamax}{8\sqrt{n}}$, since we have:
\[\hat\eta\geq \frac{\etamax}{2^K}=\frac{\etamax}{2^{\lceil \log_2(\sqrt{n})\rceil +2}}\geq \frac{\etamax}{2^{\log_2(\sqrt{n})+3}}=\frac{\etamax}{8\sqrt{n}}.\] 

Here the notation $a^{b}_{[**]}$ indicates 
$\max \{a^b, a \}$ and here and below we set $\alpha  = \minim{\frac{1}{2}}{\beta}$.

From inequality~\eqref{eq:finalesisumc}, we can derive the following two ESIs.
First, by substituting $R(A | \sZ_{\bf 0}; \cD)$ and $R(A | \sZ_{\bf 0}; \sZ_{\bf 0})$ and $\eta := n \etamin$ and rearranging,
we have for every $\eta \leq \sqrt{n}\etamax / 8$ that
\begin{multline}\label{eq:gebound}
L(A | \sZ_{\bf 0}; \cD) - L(A | \sZ_{\bf 0}; \sZ_{\bf 0}) 
\stochleq^{\sZ_{\bf 0}}_{\eta} \\
\alpha \cdot R(A | \sZ_{\bf 0}; \cD) +  4  \cdot  \left(\frac{\ex{\sZ_{\bf 1}}{\dkl{A | \sZ_{\bf 0}}{\pi|  \sZ}} + \llog n }{n \etamax}\right)^{1/(2-\beta)}_{[**]}
+ L(f^*; \cD) - L(f^*; \sZ_{\bf 0}) 
 \end{multline}
Second, by rearranging and multiplying by $\alpha/(1-\alpha)$,  (\ref{eq:finalesisumc}) also gives 
\begin{align}\label{eq:exriskbound}
\alpha R(A | \sZ_{\bf 0}; \cD) \stochleq^{\sZ_{\bf 0}}_{\eta \left(1-\alpha \right)/\alpha} 
2 \alpha \cdot \left( R(A |\sZ_{\bf 0}; \sZ_{\bf 0} ) 
+ 4  \cdot  \left(\frac{\ex{\sZ_{\bf 1}}{\dkl{A | \sZ_{\bf 0}}{\pi|  \sZ}} + \llog n }{n \etamax}\right)^{1/(2-\beta)}_{[**]}\right),
 \end{align}
where we used that $\alpha \leq 1/2$ hence $\alpha/\left(1-\alpha \right) \leq 1$ and 
the fact that, straightforwardly,  $U \stochleq_{\eta} 0 \Rightarrow c U  \stochleq_{\eta /c} 0$.
\commentout{Using this fact again we can multiply the first ESI (\ref{eq:gebound}) with $1- \alpha$ and the second ESI (\ref{eq:exriskbound}) by $\alpha$ to get for all $\eta \leq \sqrt{n} \etamax/8$, respectively
\begin{multline}\label{eq:geboundb}
(1-\alpha ) (L(A | \sZ_{\bf 0}; \cD) - L(A | \sZ_{\bf 0}; \sZ_{\bf 0}) ) 
\stochleq^{\sZ_{\bf 0}}_{\eta/ (1-\alpha)} \\
(1- \alpha ) \left( \alpha \cdot R(A | \sZ_{\bf 0}; \cD) +  4  \cdot  \left(\frac{\ex{\sZ_{\bf 1}}{\dkl{A | \sZ_{\bf 0}}{\pi|  \sZ}} + \llog n }{n \etamax}\right)^{1/(2-\beta)}_{[**]}
+  L(f^*; \cD) - L(f^*; \sZ_{\bf 0}) \right) 
 \end{multline}
and  
 \begin{align}\label{eq:exriskboundb}
\alpha R(A | \sZ_{\bf 0}; \cD) \stochleq^{\sZ_{\bf 0}}_{\eta (1-\alpha)/\alpha} 
\alpha \left( 2 \cdot \left( R(A |\sZ_{\bf 0}; \sZ_{\bf 0} ) 
+ 4  \cdot  \left(\frac{\ex{\sZ_{\bf 1}}{\dkl{A | \sZ_{\bf 0}}{\pi|  \sZ}} + \llog n }{n \etamax}\right)^{1/(2-\beta)}_{[**]}\right) \right),
 \end{align}}
We want to combine these two ESIs, while also replacing the final term $L(f^*; \cD) - L(f^*; \sZ_{\bf 0})$ in (\ref{eq:gebound}). For this we note that Hoeffding's Lemma in ESI notation combined with the ESI chain rule (Proposition~\ref{prop:esichainrule}) for i.i.d.\ random variables  immediately gives $  n  (L(f^*; \cD) -  L(f^*; \sZ_{\bf 0})) \stochleq_{\eta'} 2  n \eta' $ for all $\eta' > 0$, hence also  $   L(f^*; \cD) -  L(f^*; \sZ_{\bf 0}) \stochleq_{n \eta'} 2  \eta' $ and hence substituting $\eta:= \eta'n$, 
\begin{equation}\label{eq:hoeffdings}
 L(f^*; \cD) - L(f^*; \sZ_{\bf 0})  \stochleq_{\eta}   \frac{2\eta}{n}.
\end{equation} Chaining ESIs (\ref{eq:gebound}),  (\ref{eq:exriskbound}) and (\ref{eq:hoeffdings}), 
using Proposition~\ref{prop:esichainrule}(a), now gives, for all $\eta \leq \sqrt{n} \etamax/8$, 
\begin{multline}\label{eq:finalesi}
 L(A | \sZ_{\bf 0}; \cD) - L(A | \sZ_{\bf 0}; \sZ_{\bf 0})
 \stochleq^{\sZ_{\bf 0}}_{\eta (1- \alpha)/(2-\alpha) }  \\
 \minim{1}{2\beta} \cdot R(A | \sZ_{\bf 0}; \sZ_{\bf 0}) + 8  \cdot  \left(\frac{\ex{\sZ_{\bf 1}}{\dkl{A | \sZ_{\bf 0}}{\pi|  \sZ}} + \llog n }{n \etamax}\right)^{1/(2-\beta)}_{[**]}
 +  \frac{2 \eta}{n} .
\end{multline}
Since, by $0 \leq \alpha \leq 1/2$,  $(1-\alpha)/(2-\alpha) \geq 1/3$,
the result follows substituting $\eta$ in place of $\eta/3$.  
\end{proof}

\subsection{Applications}\label{sec:applications}
In this section, we demonstrate some applications of Theorem~\ref{th:main}, providing classes $\cF$ for which standard PAC-Bayesian bounds are suboptimal or difficult to obtain, but the almost exchangeable priors conditioned on supersamples make them straightforward. We note that the settings are slight extensions of examples already covered by \citet{Audibert04} and \citet{SteinkeZ20} in the non-fast rate setting; the added benefit is the fast-rate treatment allowed by Theorem~\ref{th:main} and its extension for Gibbs posteriors in Example~\ref{ex:gibbs}. 
For starters, the following observation (proof omitted) allows us to mix almost exchangeable priors and to construct them from standard priors: 
\begin{proposition}
Let $W$ be any standard distribution on $\cF$ independent of the data, i.e. $W | \langle \sz \rangle = W | \langle \sz'\rangle $ for all $\sz,\sz' \in \cZ^{2n}$. Then $W$ is also an almost exchangeable prior. 
Further, let $\{W_k : k \in {\mathbb N} \}$ denote  a countable set of almost exchangeable priors and let $\rho$ be a probability mass function on ${\mathbb N}$. Then $W$ defined by $W \mid \langle \sz \rangle = \sum_{k \in {\mathbb N} } \rho(k) \cdot W_k \mid \langle \sz \rangle$ is an almost exchangeable prior as well.  
\end{proposition}
\subsubsection{VC classes}
In this section, $\cZ=\cX\times\{0,1\}$ and $\cF=\{f:\cX\rightarrow\{0,1\}\}$ is a hypothesis class with VC dimension $d$. We work with the 0-1 loss $\ell : \cF \times (\cX \times \{0,1\}) \to \{0,1\}$ defined by $\ell(f,(x,y))=0 \iff f(x)=y$.

\begin{theorem}\label{th:vc-kl}
 Let $\cF=\{f:\cX\rightarrow\{0,1\}\}$ be a hypothesis class with VC dimension $d$ and let $\cZ=\cX\times\{0,1\}$.
 There exists a deterministic Empirical Risk Minimization algorithm $A:\cZ^*\rightarrow \cF$ for $0/1$ loss and an almost exchangeable prior $\pi$, such that, for any $\sz_{\bf 0}, \sz_{\bf 1}\in\cZ^n$,
$$\dkl{A|\sz_{\bf 0}}{\pi|\langle\sz_{\bf 0}, \sz_{\bf 1}\rangle}\leq d\log(2n).$$
\end{theorem}
The theorem can be proven by following the same steps as the proof of~\citet[Theorem 4.12]{SteinkeZ20}; we provide its proof in Appendix~\ref{app:proofs}.

\begin{example}[Thresholds]\label{ex:threshold}
{ \rm Consider the set of threshold functions $\cT=\{f_t:\N \rightarrow \{0,1\} : t\in  \N \cup \{\infty\} \}$, where $f_t(x)=1 \iff x\geq t$. Let $\ell$ be the 0/1 loss satisfying $\ell(f,(x,y)) = 0 \iff f(x)=y$. Let $A:(\N \times \{0,1\})^n\rightarrow \cT$ be a learning algorithm that outputs the smallest optimal threshold -- i.e., $A|z=f_{\min\{x: f_x\in\argmin_{f \in \cT}{\ell(f,z)}\}\cup\{\infty\}}$. It is straightforward to see that $A$ is an ERM that satisfies the global consistency property from the proof of Theorem~\ref{th:vc-kl}.
Since $\cT$ has VC dimension $d=1$, there exists an almost exchangable prior $\pi$ such that $\dkl{A|\sz_{\bf 0}}{\pi|\langle\sz_{\bf 0}, \sz_{\bf 1}\rangle}\leq \log(2n)$ for all $\sz$.
Now suppose that we have a distribution $\cD$ with random label noise -- i.e., there is some $t^*$ such that each data point $X_i$ is sampled from an arbitrary $\cD_{\mathcal{X}}$ and, given $X_i$, $Y_i=f_{t^*}(X_i)$ with probability $1-p$ and $Y_i=1-f_{t^*}(X_i)$ with probability $p$ for some $0 < p < 1/2$. This implies the Massart condition and hence the Bernstein condition with $\beta=1$ and $B$ depending on $p$ \citep{vanerven2015fast}. Still, the empirical error of ERM does not go to $0$ with $n$ due to the label noise. Therefore, standard PAC-Bayes bounds (\ref{eq:firstgebound}) are of order $\sqrt{\kl/n}$, whereas Theorem~\ref{th:main} gives a fast rate of order $(\log n)/n$.}
\end{example}

\subsubsection{Compression Scheme Priors}
The following extends the notion of a compression scheme due to \citet{LittlestoneW86}.
\begin{definition}[Compression Scheme Prior]\label{def:compressionscheme}
We call a  data-dependent distribution  $W:\cZ^n\rightarrow \Delta(\cF)$ a {\em compression scheme prior of size $k$} if we can write $W|z = W_2|(A_1|z)$ for all $z$, where
\begin{enumerate}
    \item $A_1:\cZ^n\rightarrow \cZ^k$ is a ``compression algorithm'' which given a sample  $z\in\cZ^n$ selects a subset $i_1, \ldots, i_k \in [n]$ and returns  $(z_{i_1}, \ldots, z_{i_k})\in\cZ^k$ and
    \item $W_2:\cZ^k \rightarrow \Delta(\cF)$ is any function. 
\end{enumerate}
For $k=0$, we say that $W$ is a compression scheme prior of size $0$ iff it outputs a fixed distribution.  
\end{definition}
 \commentout{$A:\cZ^n\rightarrow \cF$ has a {\em compression scheme of size $k$} if we can write $A|z =A_2|(A_1|z)$ for all $z$, where
\begin{enumerate}
    \item $A_1:\cZ^n\rightarrow \cZ^k$ is a ``compression algorithm'' which given a sample $z\in\cZ^n$ selects a subset $i_1, \ldots, i_k \in [n]$ and returns  $A_1|z=(z_{i_1}, \ldots, z_{i_k})\in\cZ^k$ and
    \item $A_2:\cZ^k\rightarrow \cF$ is an arbitrary ``encoding algorithm.''
\end{enumerate}

}

\begin{theorem}\label{th:compression-kl}
Let $W:\cZ^n\rightarrow \Delta(\cF)$ be a compression scheme prior  of size $k\geq 0$ and $A: \cZ^n \rightarrow \Delta(\cF)$ be an arbitrary possibly randomized learning algorithm. Then there exists an almost exchangeable prior $\pi$, such that for all  $\sz_{\bf 0}, \sz_{\bf 1}\in\cZ^n$,
\begin{equation}\label{eq:compression}
\dkl{A|\sz_{\bf 0}}{\pi|\langle\sz_{\bf 0}, \sz_{\bf 1}\rangle}\leq 
\dkl{A|\sz_{\bf 0}}{W| \sz_{\bf 0}} +
k\log(2n).\end{equation}
\end{theorem}
\begin{proof}[Proof of Theorem~\ref{th:compression-kl}]
Let $W=W_2|(A_1|z)$ be a compression scheme prior and let $\langle \sz\rangle =\langle \sz_{\bf 0}, \sz_{\bf 1}\rangle$. We choose the conditional prior distribution as 
$$\pi(f|\langle\sz\rangle)=\frac{\sum_{z^k\in K(\sz)} \cP_{W_2|z^k}(f)}{\binom{2n}{k}},$$
where we denote by $K(z)$ the set of all subsets of $z$ of size $k$. Observe that $\pi$ is indeed an almost exchangeable prior. It holds that
\begin{align*}
\dkl{A|\sz_{\bf 0}}{\pi|\langle\sz_{\bf 0}, \sz_{\bf 1}\rangle}
& =\ex{f\sim A|\sz_{\bf 0}}{\log\frac{\cP_{A|\sz_{\bf 0}}(f)}{\pi(f|\langle\sz\rangle)}} \\
& =\ex{f\sim A|\sz_{\bf 0}}{\log\frac{\cP_{A|\sz_{\bf 0}}(f)\cdot \binom{2n}{k}}{\sum_{z^k\in K(\sz)} \cP_{W_2|z^k}(f)}} \\
& \leq \ex{f\sim A|\sz_{\bf 0}}{\log\frac{\cP_{A|\sz_{\bf 0}}(f)\cdot \binom{2n}{k}}{\cP_{W_2|(A_1|\sz_{\bf 0})}(f)}} \\
& = \ex{f\sim A|\sz_{\bf 0}}{\log\frac{\cP_{A|\sz_{\bf 0}}(f)}{\cP_{W|\sz_{\bf 0}}(f)}} +\log \binom{2n}{k} \\
& \leq \dkl{A|\sz_{\bf 0}}{W|\sz_{\bf 0}} + k\log(2n),
\end{align*}
where the first inequality holds since $A_1|\sz_{\bf 0}\in K(\sz)$ which implies that $\sum_{z^k\in K(\sz)} \cP_{W_2|z^k}(f)\geq \cP_{W_2|(A_1|\sz_{\bf 0})}(f)$.
The last inequality follows by the common bound $\binom{2n}{k}\leq (2n)^k$.
\end{proof}

In the case that we choose a size $k$ compression scheme prior $W$ that, upon each input, puts all its mass on a single $\hat{f} \mid \sZ_{\bf 0} \in \cF$, and we choose $A$ to be the deterministic 
learning algorithm that is {\em equal\/} to $W$, then $A$ has a compression scheme of size $k$ in the original sense of \citet{LittlestoneW86} and its $\kl$ complexity will by Theorem~\ref{th:compression-kl} be bounded by 
$k \log (2n)$. Our generalization allows us to choose an algorithm $A$ different from $W$ that might, 
for example, base its output on the whole dataset and not just the $k$ points selected by $A_1$ `inside' $W$. An example algorithm with pleasant properties is the {\em Gibbs algorithm\/} with $W$ as a prior.  

\begin{example}{{\rm \bf (Gibbs Algorithm based on Compression Scheme Prior)}}
\label{ex:gibbs} {\rm 
The {\em Gibbs\/} or {\em generalized Bayes\/} learning algorithm (see, e.g., \citet{Alquier20,GrunwaldM20,zhang2006information}) $A_{\textsc{Gibbs}}: \cZ^n \rightarrow \Delta(\cF)$ with (possibly data-dependent) learning rate $\hat\eta$ based on data-dependent prior distribution $W$ is defined in terms of its posterior density (Radon-Nikodym derivative) relative to $W$, as 
$$
\frac{d (A_{\textsc{Gibbs}} | \sZ_{\bf 0})}{d (W | \sZ_{\bf 0})} (f)  \propto 
\exp(- \hat\eta n R(f; \sZ_{\bf 0}) ).
$$
This is the standard definition of the Gibbs algorithm relative to prior distribution $W \mid \sZ_{\bf 0}$. A modification of the proof of Theorem~\ref{th:main}, sketched in Appendix~\ref{app:proofs}, gives the following corollary: {\em if we set $A$ to the Gibbs algorithm relative to size $k$ compression scheme prior $W$, and $A'$ to any other algorithm, we have, with the same abbreviations as in Theorem~\ref{th:main}},
\begin{multline}
L(A_{\textsc{Gibbs}}| \sZ_{\bf 0} ; \cD) 
\stochleq^{\sZ_{\bf 0}}_{\eta}  \\ 
L(A' | \sZ_{\bf 0} ;  \sZ_{\bf 0})  +  \minim{1}{2\beta} \cdot {R(A' | \sZ_{\bf 0} ;  \sZ_{\bf 0})} +  8 \cdot  \left(\frac{\dkl{A' | \sZ_{\bf 0}}{W|\sZ_{\bf 0} }+ O (k \log n) }{n\etamax}  \right)^{1/(2-\beta)}_{[**]} + \frac{6 \eta}{n}. \nonumber
\end{multline}
In particular, if $A'$ is set to an ERM, the sum of the first two terms on the right is upper bounded by
$L(A_{\textsc{Gibbs}}| \sZ_{\bf 0} ; \sZ_{\bf 0})$
again and, under a Bernstein condition, we get a fast rate for the Gibbs algorithm as well, although the complexity term is taken relative to ERM rather than Gibbs.}
\end{example}

\section{Conclusion and Future Work}\label{sec:conclusion}
We have shown how to extend PAC-Bayesian and Mutual Information Bounds to a fast-rate conditional version which allows us to handle arbitrary VC classes. One point which remains open for future research is the fact that, unless we use ERM and we deal with losses like the squared error for which we know the $\beta$ for which the Bernstein condition holds in advance, the bound (\ref{eq:mainpre}) is not empirical (observable from data only, without knowing $\cD$ or $f^*$). \citet{MhammediGG19} do provide an empirically observable bound that achieves fast rates, by replacing $f^*$ by an estimator based on part of the training data only (a technique called  {\em de-biasing\/} by Y. Seldin) and by replacing the $O((\kl/n)^{1/(2-\beta)})$ term by an  empirical 
variance-like term that goes to $0$ at the right rate if a Bernstein condition holds but can be calculated without knowing $\beta$. It seems likely that our bound can also be made fully empirical, for arbitrary learning algorithms and losses rather than just ERM and curved losses. Whether this is really the case will be sorted out in future work.  
Another interesting open question is whether a similar bound holds for unbounded but sub-Gaussian losses; see the discussion underneath Lemma~\ref{lem:esiindividual}. 

\section*{Acknowledgements}
We thank the reviewers of COLT for useful comments on the presentation of our manuscript. LZ was supported by a Facebook Fellowship.
PG was supported by the Dutch Research Council (NWO) via research programme 617.001.651. This work began when TS and LZ were at IBM Research in Almaden.

\DeclareRobustCommand{\VANDER}[3]{#3}
\bibliographystyle{plainnat}
\bibliography{bib,master}

\appendix
\section{Glossary}\label{app:glossary}

\begin{table}[H]
    \centering
    \begin{tabular}{l  l}
    \toprule
       \textbf{Notation}  & \textbf{Description}  \\ \toprule
        $\cD$ & Probability distribution over $\cZ$  \\ \midrule
        $Z$ & i.i.d.\ sample of size $n$: $Z=(Z_1, \ldots, Z_n)\sim \cD^n$  \\ \midrule
        $(\cD, \ell, \cF)$ & Learning problem for distribution $\cD$, loss function $\ell$, and set of hypotheses $\cF$   \\ \midrule
        $\ell(f;z)$ & Empirical loss of $f$ on sample $z\in\cZ^n$: $\ell(f;z)=\frac{1}{n}\sum_{i=1}^n\ell(f;z_i)$   \\ \midrule
        $\ell(f;\cD)$ & True loss of $f$: $\exinline{Z'\sim\cD}{\ell(f;Z')}$   \\ \midrule
        $f^*$ & True loss minimizer within $\cF$: $\ell(f^*;\cD)=\inf_{f\in\cF}\ell(f;\cD)$   \\ \midrule
        $A$ & (Possibly randomized) learning algorithm: $A:\bigcup_{i=1}^n\cZ^i \rightarrow \Delta(\cF)$   \\ \midrule
        $A|Z$ & Posterior distribution of output of $A$ on input $Z\sim\cD^n$   \\ \midrule
        $L(F;z)$ & Empirical loss of $F\in\Delta(\cF)$ on sample $z\in\cZ^n$: $L(F;z)=\exinline{f\sim F}{\ell(f;z)}$  \\ \midrule
        $L(F;\cD)$ & True loss of $F\in\Delta(\cF)$: $L(F;\cD)=\exinline{f\sim F}{\ell(f;\cD)}$  \\ \midrule
        $R(F;z)$ & Empirical excess risk of $F\in\Delta(\cF)$ on sample $z\in\cZ^n$: $R(F;z)=\exinline{f\sim A}{r(f;z)}$  \\ \midrule
        $R(F;\cD)$ & True excess risk of $F\in\Delta(\cF)$: $R(F;\cD)=\ex{f\sim A}{r(f;\cD)}$  \\ \midrule
        $\sZ$ & Supersample $\sZ=\left((\sZ_{1,0}, \sZ_{1,1}), \ldots, (\sZ_{n,0}, \sZ_{n,1})\right)^\top\sim\cD^{n\times 2}$  \\ \midrule
        $S$ & Random selector vector $S\sim \Ber(1/2)^n$  \\ \midrule
        $\sZ_S$ & Subset of $\sZ$ indexed by $S\in\{0,1\}^n$: $\sZ_S=(Z_{1,S_1}, \ldots, Z_{n,S_n})^\top\in\cZ^n$  \\ \midrule
        $\langle \sZ\rangle$ & List of \emph{unordered} pairs of $\sZ$: $\langle \sZ\rangle = \langle \sZ_{\bf 0},\sZ_{\bf 1}\rangle= (\{\sZ_{1,0}, \sZ_{1,1}\},\ldots,\{\sZ_{n,0},\sZ_{n,1}\})^\top$  \\ \bottomrule
    \end{tabular}
    \caption{Notation}
    \label{tab:notation}
\end{table}

\section{Omitted proofs}\label{app:proofs}

\subsection{Linearized version of Bernstein Condition}
It will be convenient to work with the following {\em linearized version\/} of the Bernstein condition. It is an extension of a well-known result that has appeared in previous work, e.g. in~\citep{WooterGE16}. We restate it here for convenience.
\begin{proposition}[Restatement of Proposition~\ref{prop:linearizedBernstein}]\label{prop:linearizedBernstein-app}
Suppose that $(\cD, \ell, \cF)$ satisfies the {$(B,\beta^*)$-Bernstein condition} for $\beta^* \in [0,1]$. Pick any $c > 0, \eta < 1/(2 B c)$. Then for all $0 < \beta \leq \beta^*$ and for all $f\in\cF$:
$$c \cdot \eta \ex{Z'\sim\cD}{(\ell(f;Z')-\ell(f^*; Z'))^2} \leq \minim{\frac{1}{2}}{\beta} \cdot \left(\ex{Z'\sim\cD}{\ell(f;Z')-\ell(f^*; Z')}\right)  ~ 
+ (1- \beta)   \cdot (2 Bc \eta)^{\frac{1}{1-\beta}} 
$$ 
\end{proposition}

\begin{proof}[Proof of Proposition~\ref{prop:linearizedBernstein}]
We first prove the proposition for $0 < \beta < 1$. For any  $\eta > 0$, $B' > 0$, let $g(x)=B'\eta x^{\beta} - x$, for $x>0$. We have
$$\max_{x>0} \{g(x)\}=\max_{x>0} \{B'\eta x^{\beta} - x \} = (1-\beta) {\beta}^{\beta/(1-\beta)} \cdot \left( B' \eta \right)^{1/(1-\beta)},$$
since $g'(x)=0$ for $x = (B' \eta  \beta)^{1/(1-\beta)}$ and $g''(x)<0$ for all $x>0$.
Hence, for all  $0< a \leq  1$ and $c > 0$, by setting $B'= B c /a$,
we have: 
\begin{align}\label{eq:maxg}
\max_{x>0}\{B c \eta x^{\beta}-a x\} & = \max_{x>0} \{a \cdot g(x)\} \nonumber \\
&= a \cdot (1-\beta) \cdot \beta^{\beta/(1-\beta)} \cdot \left( \frac{B c}{a}\eta \beta\right)^{1/(1-\beta)} \nonumber \\
&=  a^{- \beta/(1-\beta)} \cdot (1-\beta) \cdot \beta^{\beta/(1-\beta)} \cdot  \left( B c \eta \right)^{1/(1-\beta)}
\end{align}
Now, by assumption, the $(B,\beta^*)$-Bernstein condition holds for $\beta^*\geq \beta$.
Since the excess risk $R(f;\cD)=\ex{Z'\sim\cD}{\ell(f;Z')-\ell(f^*; Z')}\in [0,1]$, the $(B,\beta)$-Bernstein condition also holds, which implies that
\begin{equation*}
    c\eta\ex{Z'\sim\cD}{\left(\ell(f;Z')-\ell(f^*; Z')\right)^2} \leq Bc\eta \left(\ex{Z'\sim\cD}{\ell(f;Z')-\ell(f^*; Z')}\right)^\beta.
\end{equation*}

We now apply~\eqref{eq:maxg} with  $x=\ex{Z'\sim\cD}{\ell(f;Z')-\ell(f^*; Z')}$ and $a = \minim{\frac{1}{2}}{\beta}$ in the above inequality, establishing that
\begin{equation*}
    c\eta\ex{Z'\sim\cD}{\left(\ell(f;Z')-\ell(f^*; Z')\right)^2} \leq a\ex{Z'\sim\cD}{\ell(f;Z')-\ell(f^*;Z')}+a^{- \frac{\beta}{1-\beta}} \cdot (1-\beta) \cdot \beta^{\frac{\beta}{1-\beta}} \cdot  \left( B c \eta \right)^{\frac{1}{1-\beta}}.
\end{equation*}
Bounding the last term of the RHS by $(1-\beta)\cdot (2Bc\eta)^{1/(1-\beta)}$ would complete the proof for $0<\beta<1$. For this to hold, it suffices to prove that $\left(\beta/a\right)^{\beta/(1-\beta)}\leq 2^{1/(1-\beta)}$, for $0<\beta<1$. If $a=\beta$, then the inequality reduces to $1\leq 2$, which trivially holds. If $a=1/2$, then the inequality reduces to $(2\beta)^\beta\leq 2$, which also holds. 

It remains to prove the proposition for the limiting cases of $\beta=0$ and $\beta=1$. For $\beta=0$, the RHS reduces to $(2Bc\eta)$, and the inequality trivially holds by the assumption of the $(B,\beta^*)$-Bernstein condition and the trivial bound of $\left(\exinline{Z'\sim\cD}{\ell(f;Z')-\ell(f^*; Z')}\right)^{\beta^*}\leq 1$. For $\beta=\beta^*=1$, the RHS reduces to $\frac{1}{2} \exinline{Z'\sim\cD}{\ell(f;Z')-\ell(f^*; Z')}$, and the inequality also holds by the assumption of the $(B,1)$-Bernstein condition and our setting of $\eta<1/(2Bc)$.
\end{proof}

\subsection{Proof of main technical Lemma~\ref{lem:esiindividual}}
For convenience we first restate the lemma: 
\begin{lemma}[Restatement of main technical Lemma~\ref{lem:esiindividual}]\label{lem:esiindividualb}
Fix any two real numbers $r_0,r_1$ such that $|r_0|, |r_1| \leq 1$. Let $S\sim\Ber(1/2)$ and let $\bar{S} = 1- S$. Then for all $0 < \eta < 1/(1+ \sqrt{2})$, it holds that
$$
r_{\bar{S}} - r_{S} \stochleq_{\eta}  \eta \cdot C_{2,\eta} r^2_{\bar{S}}
$$
with $C_{A,\eta}$ an increasing function of $\eta$ given by $$C_{A,\eta}= 
\frac{1}{1- \eta} \cdot  \left(A + \sqrt{A} \cdot \frac{\eta}{1- \eta} \cdot c_{\sqrt{A}\eta /(1- \eta)}\right),$$ 
where $c_{\gamma}= 2 \frac{(-\log{(1-\gamma)}-\gamma)}{\gamma^2}$.
If both $r_0$ and $r_1$ have the same sign, the constant can be improved to $C_{1,\eta}$ and the result holds for all $0 < \eta < 1/2$. Since $c_{\gamma}$ is increasing and $\lim_{\gamma \downarrow 0} c_{\gamma} = 1$ the `leading constant' is given by  $\lim_{\eta \downarrow 0} C_{2,\eta} = 2$ (and $\lim_{\eta \downarrow 0} C_{1,\eta} = 1$ in case both $r_0$ and $r_1$ have the same sign).
\end{lemma}
For simplicity in the derivations, in the main text we will consider only $\eta \leq 1/4$ and use $C_{1/4}= 3.6064$ as an upper bound on $C_{2,\eta}$.
It is easy to see that the result is tight in the limit for $\eta \downarrow 0$, by considering the case $r_0 = - r_1$ and doing a second order Taylor approximation of $\exinline{S}{\exp(\eta( r_{\bar{S}}- r_S))}$ around $\eta =0$. The result is only proven for $r_0, r_1$ with $|r_0|, |r_1| \leq 1$, and (since $c_{\gamma}$ tends to $\infty$ as $\gamma \uparrow 1$), the bound becomes void for $\eta \geq 1/(1+ \sqrt{2})$. Yet, as is straightforward to show by inspecting the formulas, for general $r_0 \leq r_1< \infty$ we still have $r_{\bar{S}} - r_{S} \stochleq_{\eta} \eta B r^2_{\bar{S}}$ for some finite $B$ as long as $r_0 > - \log 2$, with $B$ tending to infinity as $r_0 \downarrow - \log 2$; it is just not so easy any more to give a crisp bound. \\ \ \\

The proof crucially makes use of the following \emph{un-expected Bernstein inequality}~(originally due to \cite{fan2015exponential}, our presentation follows \cite{MhammediGG19} who gave it its name):
\begin{lemma}[{Un-expected Bernstein Inequality \cite[Lemma 13(a)]{MhammediGG19}}]\label{lem:unexpectedBernstein}
Let $U$ be a random variable bounded from above by $b>0$ almost surely, and let $\theta(u)=(-\log(1-u)-u)/u^2$. For all $0<\eta<1/b$, we have
\[\ex{}{U}-U \stochleq_\eta \frac{1}{2}\eta c_{\eta}\cdot U^2 ~\text{ for all }c_{\eta}\geq 2\cdot \theta(\eta b).\]
\end{lemma}

\begin{proof}[Proof of Lemma~\ref{lem:esiindividual}]
We only prove the case for general $r_0, r_1$ with $|r_0|, |r_1| \leq 1$. The improved bound for $r_0$ and $r_1$ of the same sign can be proven by following exactly the same steps as below, where the term $(r_0-r_1)^2$ in the derivation of (\ref{eq:improve}) is bounded by $r_1^2 + r_2^2$ instead of $2 r_1^2 + 2 r_0^2$.

Fix $\lambda > 0$ and let $x \in \R$.
The well-known $\cosh$-inequality states that 
$(1/2) \exp(\lambda x) +(1/2) \exp(-\lambda x) \leq \exp(\lambda^2 x^2/2)$. 
Now fix $x$ and let $Y$ be a Rademacher RV such that $P(Y=x) = P(Y= -x) =1/2$. By definition, for all $\lambda>0$, $\ex{Y}{\exp(\lambda Y)}=(1/2)\exp(\lambda x)+ (1/2)\exp(-\lambda x)$. Therefore, by the cosh-inequality, we have for all $\eta > 0, A > 0$, and letting $\lambda=A\eta$, that
\begin{align}\label{eq:cosh}
    Y \stochleq_{A\eta } \frac{1}{2} A\eta  x^2.
\end{align}
Now, let $U$ be a RV such that $U\in[0,1]$.
Then by the {\em un-expected Bernstein inequality\/} of~\cite{MhammediGG19} (Lemma~\ref{lem:unexpectedBernstein}) we have, for all $0 < \eta < 1$, 
$$
    \ex{U}{U} \stochleq_{\eta} U + \frac{1}{2} \eta c_{\eta} U^2,
$$
for $c_{\eta}= 2 \frac{(-\ln{(1-\eta)}- \eta)}{\eta^2}$. 
Since $U\geq 0$, it follows that for all $0 < \eta < 1$, 
$$
    \ex{U}{U} \stochleq_{\eta} (1+ \eta c_{\eta}/2) U.
$$
Hence 
\begin{align}\label{eq:unbern}
2A \eta \ex{U}{U} \stochleq_{1/2A} A \eta  (2+ \eta c_{\eta}) U.
\end{align}

Note that with $x= r_{1}-r_{0}$, $r_{\bar{S}}-r_S$ is a Rademacher RV such that
$
P(r_{\bar{S}} -r_{S} = x) =
P(r_{\bar{S}} -r_{{S}} = -x) = \frac{1}{2}.
$
Thus, by~\eqref{eq:cosh}, we have that for all $\eta > 0$, $A>0$,
\begin{align}
 r_{\bar{S}} - r_S & \stochleq_{A \eta}
 \frac{1}{2} A \eta \cdot (r_0 - r_1)^2 \nonumber \\
&\leq A \eta \cdot \frac{1}{2} (2 r_1^2 + 2 r_0^2)  \tag{since $(x-y)^2 \leq 2 x^2 + 2 y^2$}  \nonumber \\
& = 2 A \eta \cdot \left(\ex{S'}{r^2_{\bar{S}'}}\right).\label{eq:improve} 
  \end{align}
Since $r^2_{\bar{S}}\in[0,1]$, we also apply~\eqref{eq:unbern} to $r^2_{\bar{S}}$. We then have that, for $\eta < 1$,
$$
2 A \eta \cdot \left(\ex{S'}{r^2_{\bar{S}'}}\right) \stochleq_{A \eta}^{S'} A \eta (2 + \eta c_{\eta} ) r^2_{\bar{S}}.
$$
Now for arbitrary (possibly dependent) RVs $X,Y,Z$ we have $X \stochleq_{A \eta} Y, Y \stochleq_{1/2A} Z \Rightarrow X \stochleq_{\bar\eta} Z$, where $\bar\eta  = (1/(A \eta) + 2A)^{-1} =  A\eta/ (1+2A^2 \eta)$
(by Proposition~\ref{prop:esichainrule}). Combining the above two ESIs implies that
$$
 r_{\bar{S}} - r_S \stochleq_{\bar\eta}
 A \eta (2 + \eta c_{\eta} ) r^2_{\bar{S}}.
$$
This bound holds for all $0 < \eta < 1$ and arbitrary $A> 0$. We want this bound to hold for as large $\bar\eta$ as possible. Since $\eta = \bar\eta/(A (1- 2A \bar\eta))$ is an increasing function of $\bar\eta$, the bound is valid up to all $\bar\eta<\bar\eta^*$ where $1 = \bar\eta^*/ (A(1- 2A \bar\eta^*))$. Choosing the $A$ for which $\bar\eta^*$ is maximal gives $A = 1/\sqrt{2}$, and then $\bar\eta^*=1/(1+\sqrt{2})$ and $\eta = \sqrt{2} \bar\eta/(1-\bar\eta)$. Substituting $\eta$ and $A$ in the previous ESI we now get, for $0 < \bar\eta < 1/(1+ \sqrt{2})$,
\begin{align*}
 r_{\bar{S}} - r_S & \stochleq_{\bar\eta}
 \frac{\bar\eta}{1-\bar\eta} \cdot  \left(
  2 +  \sqrt{2}\frac{\bar\eta}{1- \bar\eta} \cdot  c_{\sqrt{2} \bar\eta/(1- \bar\eta)}
 \right) \cdot r^2_{\bar{S}} =  \bar\eta \cdot C_{\sqrt{2},\bar\eta} \cdot r^2_{\bar{S}}
 \end{align*}
and the result follows.
\end{proof}

\subsection{Improved in-expectation bound - `Variation' of Theorem~\ref{th:main}}
\begin{corollary}\label{cor:inexpect-app}{\rm (`{\bf Variation} of Theorem~\ref{th:main}' - Restatement of Corollary~\ref{cor:inexpect})}
Consider the setting and notation of Theorem~\ref{th:main}. For all $\beta\in[0,\beta^*]$, it holds that
\begin{multline}\label{eq:finalinexpect-app}
 \ex{\sZ_{\bf 0}}{L(A | \sZ_{\bf 0}; \cD) - L(A | \sZ_{\bf 0}; \sZ_{\bf 0})}
 \leq \\
 \minim{1}{2\beta} \cdot \ex{\sZ_{\bf 0}}{R(A | \sZ_{\bf 0}; \sZ_{\bf 0})} + 4 \cdot \left( 
 \frac{\ex{\sZ_{\bf 0}, \sZ_{\bf 1}}{\dkl{A | \sZ_{\bf 0}}{\pi|  \langle \sZ_{\bf 0}, \sZ_{\bf 1} \rangle}}}{n\etamax}\right)^{\frac{1}{2-\beta}}_{[**]}.
\end{multline}
\end{corollary}
The proof follows by a few modifications of the proof of the main Theorem~\ref{th:main}.
\begin{proof}[Proof Sketch]
The proof would be the same up to and including the derivation of inequality~\eqref{eq:finalesisum}, where $\eta<\etamax$ is not random. At this step, we can weaken this ESI to an in-expectation inequality, subsequently derive and add the equivalent of inequalities~\eqref{eq:gebound} and~\eqref{eq:exriskbound}, to yield
\begin{multline*}
 \ex{\sZ_{\bf 0}}{L(A | \sZ_{\bf 0}; \cD) - L(A | \sZ_{\bf 0}; \sZ_{\bf 0})}
 \leq \\
 \minim{1}{2\beta} \cdot \ex{\sZ_{\bf 0}}{R(A | \sZ_{\bf 0}; \sZ_{\bf 0})} + 2  \cdot  \left(\left( \frac{\eta}{\etamax} \right)^{\frac{1}{1-\beta}} + 
 \frac{\ex{\sZ_{\bf 0}, \sZ_{\bf 1}}{\dkl{A | \sZ_{\bf 0}}{\pi|  \langle\sZ\rangle}}}{n \eta}\right).
\end{multline*}
By differentiation, we choose $\eta=\minim{\etamax}{ (1-\beta)^{\frac{1-\beta}{2-\beta}}\etamax^{\frac{1}{2-\beta}}\left(\frac{\ex{\sZ_{\bf 0}, \sZ_{\bf 1}}{\dkl{A | \sZ_{\bf 0}}{\pi|  \sZ}}}{n\etamax}\right)^{\frac{1-\beta}{2-\beta}}}$ to minimize the sum of the last two terms of the RHS of the inequality, which gives the improved in-expectation bound:
\begin{equation*}
 \ex{\sZ_{\bf 0}}{L(A | \sZ_{\bf 0}; \cD) - L(A | \sZ_{\bf 0}; \sZ_{\bf 0})}
 \leq 
 \minim{1}{2\beta} \cdot \ex{\sZ_{\bf 0}}{R(A | \sZ_{\bf 0}; \sZ_{\bf 0})} + 4 \cdot \left( 
 \frac{\ex{\sZ_{\bf 0}, \sZ_{\bf 1}}{\dkl{A | \sZ_{\bf 0}}{\pi|  \sZ}}}{n\etamax}\right)^{\frac{1}{2-\beta}}_{[**]},
\end{equation*}
where $a^b_{[**]}=\max\{a^b,a\}$.
\end{proof}

\commentout{\subsection{Proof of CMI bound (Corollary~\ref{cor:cmi})}
\begin{proof}[Proof of Corollary~\ref{cor:cmi}]
Let $\sZ=(\sZ_{\bf 0}, \sZ_{\bf 1})$. We focus on the $\kl$ divergence in the bound~\eqref{eq:finalinexpect}:
$$
\ex{\sZ_{\bf 0}, \sZ_{\bf 1}}{\dkl{A | \sZ_{\bf 0}}{\pi|  \langle \sZ_{\bf 0}, \sZ_{\bf 1} \rangle}} 
=  \ex{S, \sZ_{\bf 0}, \sZ_{\bf 1}}{\dkl{A | \sZ_{\bf 0}}{\pi|  \langle \sZ_{\bf 0}, \sZ_{\bf 1} \rangle}}
= \ex{S, \sZ}{{\dkl{A | \sZ_{S}}{\pi|  \langle \sZ_{S}, \sZ_{\bar{S}} \rangle}}}
$$
The first equality holds since $S$ is independent of $\sZ_{\bf 0}, \sZ_{\bf 1}$. The second equality holds because the distributions of $\sZ_{S}$, $\sZ_{\bar{S}}, \sZ_{\bf 0}, \sZ_{\bf 1}$ are all identical to $\cD^n$ and $\pi$ is almost exchangeable. 
We choose $\pi=\ex{S'}{A | \sZ_{S'}}$ for $S' \sim\Ber(1/2)^n$. Notice that $\pi$ is indeed almost exchangeable. We now have
$$
\ex{S, \sZ}{{\dkl{A | \sZ_{S}}{\ex{S'}{A|\sZ_{S'}}}}}
=\ex{\sZ}{\ex{S}{\dkl{A | \sZ_{S}}{\ex{S'}{A|\sZ_{S'}}}}}
=\ex{\sZ}{I^{\sZ}(A|\sZ_S;S)}
=\CMI{A}{\cD}.
$$
Combining the two equations and substituting the term in inequality~\eqref{eq:finalinexpect} completes the proof.
\end{proof}}

\subsection{Proof of Theorem~\ref{th:vc-kl} (VC classes)}
First, we formally define the global consistency property. Here we abuse notation by interchanging between $(\cX \times \cY)^n$ and $\cX^n \times \cY^n$. That is, we refer to $(x,y)\in (\cX \times \cY)^n$ when we mean $x \in \cX^n$ and $y \in \cY^n$. We also use (and abuse) the notation $(\cX\times \cY)^* := \bigcup_{n=0}^\infty (\cX \times \cY)^n$. Thus the notation $(x,y)\in (\cX \times \cY)^*$ means, for some $n$, we have $x \in \cX^n$ and $y \in \cY^n$.

\begin{definition}[Global Consistency Property] \label{def:gc}
Let $\cF$ be a class of functions $f : \cX \to \cY$. A deterministic algorithm $A : (\cX \times \cY)^* \to \cF$ is said to have the global consistency property if the following holds. Let $(x,y) \in (\cX \times \cY)^*$ and let $f=A|(x,y)$. We require that, for any $x' \in \cX^*$ such that $x'$ contains all the elements of $x$ (i.e., $\forall i ~ \exists j ~ x_i = x'_j$), we have $A|(x',y')=f$ whenever $y'_i = f(x'_i)$ for all $i$.
\end{definition}
Informally, this property says the following. Suppose the algorithm is run on some labelled dataset $(x,y)$ to obtain an output hypothesis $f=A|(x,y)$. If the dataset is relabelled to be perfectly consistent with $f$, then the algorithm should still output $f$. This should also hold if further examples are added to the dataset (where the additional examples are also consistent with $f$).

The proof of the theorem is split in the next two lemmata.

\begin{lemma} \label{lem:gc+vc}
Let $A : (\cX\times \{0,1\})^n \to \cF$ be a deterministic algorithm, where $\cF$ is a class of functions $f : \cX \to \{0,1\}$ with VC dimension $d$.  Suppose $A$ (appropriately extended to inputs of arbitrary size) has the global consistency property. Then for any $\sz_{\bf 0}, \sz_{\bf 1}\in\cZ^n$,
$$\dkl{A|\sz_{\bf 0}}{\pi|\langle\sz_{\bf 0}, \sz_{\bf 1}\rangle}\leq d\log(2n).$$
\end{lemma}
\begin{lemma} \label{lem:gc+erm}
Let $\cF$ be a class of functions $f:\cX \to \{0,1\}$. Then there exists a deterministic algorithm $A : (\cX \times \{0,1\})^* \to \cF$ that has the global consistency property and is an empirical risk minimizer -- that is, for all $(x,y) \in (\cX \times \{0,1\})^*$, if $f^*=A|(x,y)$, then $$\sum_i \mathbb{I}[f^*(x_i) \ne y_i] = \min_{f \in \cF} \sum_i \mathbb{I}[f(x_i) \ne y_i].$$
\end{lemma}

To prove Lemma \ref{lem:gc+vc} we invoke the Sauer-Shelah lemma:\footnote{Vapnik and Chervonenkis proved a slightly weaker bound, namely $\left|\left\{(f(x_1),f(x_2),\cdots,f(x_m)) : f \in \cF\right\}\right| \le m^{d+1}+1$ for $m>d$ \cite[Thm.~1]{VapnikC71bib}.}
\begin{lemma}[{\cite{Sauer72,Shelah72}}]\label{lem:sauer}
Let $\cF$ be a class of functions $f :\cF\to \{0,1\}$ with VC dimension $d$. For any $x=\{x_1, \cdots, x_m \} \subset \mathcal{X}$, the number of possible labellings of $x$ induced by $\cF$ is
\[\left|\left\{(f(x_1),f(x_2),\cdots,f(x_m)) : f \in \cF\right\}\right| \leq \sum_{k=0}^d {m \choose k} \leq \left\{\begin{array}{cl}  (em/d)^d & ~~\text{ if } m \geq d \\  e^2 \cdot (m/2)^d  & ~~\text{ if } m \geq 2 \\ e \cdot m^d & ~~\text{ if } m \geq 1 \end{array} \right..\]
\end{lemma} 
Here we define ${m \choose k}=0$ if $k>m$. Thus $\sum_{k=0}^d {m \choose k}=2^m$ if $m \leq d$. Note that we give three different forms of the final bound for convenience, all of which are derived from the bound 
$$\forall m \geq d ~~~ \forall x \geq 1 ~~~~~~ \sum_{k=0}^d {m \choose k} \leq \sum_{k=0}^d {m \choose k} x^{d-k} \leq \sum_{k=0}^m {m \choose k} x^{d-k} = \left(1+x^{-1}\right)^m \cdot x^{d} \leq e^{m/x} \cdot x^d.$$

\begin{proof}[Proof of Lemma~\ref{lem:gc+vc}]
Let $\sz=(\sz_{\bf 0},\sz_{\bf 1})$ be the fixed supersample and let $\sx=\{x: \exists y\in\{0,1\}, i\in[n], j\in\{0,1\} : (x,y)=\sz_{i,j}\}$ be the set of all unlabelled examples in $\sz$. We choose as an almost exchangeable prior distribution $\pi$ the following: $\pi(f)=\mathbb{I}\{\exists s\in\{0,1\}^n : A|\sz_s=f\}/|H(\sz)|$, where $H(\sz)=\{A|\sz_s \text{ for some } s\in\{0,1\}^n\}$. That is, $\pi$ is uniform over all the possible outputs of algorithm $A$ given input $\sz_s$ for some $s\in\{0,1\}^n$. Then the KL term is written as
$$ \dkl{A|\sz_{\bf 0}}{\pi|\langle\sz_{\bf 0}, \sz_{\bf 1}\rangle} = \log \frac{1}{\pi(A|\sz_{\bf 0})} = \log \frac{|H(\sz)|}{\mathbb{I}\{\exists s\in\{0,1\}^n : A|\sz_s=A|\sz_{\bf 0}\}} =\log |H(\sz)|.$$
It suffices to bound $|H(\sz)|$. 
By the global consistency property, if $A|\sz_s=f$ for some $s\in\{0,1\}^n$, then it must be that $A|(\sx,f(\sx))=f$. Therefore
$$H(\sz) \subseteq \{A|(\sx,f(\sx)) : f\in\cF\} \subseteq \{f(\sx): f\in\cF\}$$
By Lemma~\ref{lem:sauer}, the set of all the possible labellings of $\sx\in\cX^{2n}$ by $\cF$ has size at most $|\{f(\sx): f\in\cF\}|\leq (2n)^d$. Thus, $|H(\sz)|\leq (2n)^d$ and the bound of the lemma follows.
\end{proof}

Lemma~\ref{lem:gc+erm} is exaclty the same as the corresponding lemma in the proof of the CMI result of~\cite{SteinkeZ20}. We present their proof here to give a clear picture of the type of algorithm that could satisfy the lemma for our examples.

For the proof, we will invoke the well-ordering theorem~\cite{Zermelo04}:
\begin{lemma}[{\cite{Zermelo04}}]\label{lem:wellorder}
Let $\cF$ be a set. Then there exists a binary relation $\preceq$ with the following properties.
\begin{itemize}
    \item Transitivity: $~~~\forall f,g,h \in \cF ~~~ f \preceq g \wedge g \preceq h \implies f \preceq h$
    \item Totality: $~~~\forall f,g \in \cF ~~~ f \preceq g \vee g \preceq f$
    \item Antisymmetry: $~~~\forall f,g \in \cF ~~~ f \preceq g \wedge g \preceq f \iff f=g$
    \item Well-order: $~~~\forall H \subset \cF ~~ \left( ~ H \ne \emptyset ~~ \implies ~~ \exists h \in H ~~ \forall f \in H ~~ h \preceq f ~ \right)$
\end{itemize}
\end{lemma}
Let $\preceq$ be a well-ordering of $\cF$.
On a finite computer, we could simply let $\preceq$ be the lexicographical ordering on the binary representations of elements of $\cF$.

\begin{proof}[Proof of Lemma \ref{lem:gc+erm}]
An empirical risk minimizer $A : (\cX \times \{0,1\})^n \to \cF$ must have the property $$\forall (x,y) \in (\cX \times \{0,1\})^n ~~~~~ A|(x,y) \in \argmin_{f \in \cF} \ell(f,(x,y)) := \left\{f \in \cF : \ell(f,(x,y)) = \inf_{f' \in \cF} \ell(f',(x,y))\right\}.$$
However, we must also ensure that $A$ satisfies the global consistency property. The only difficulty that arises here is when the argmin contains multiple hypotheses; we must break ties in a consistent manner. (Note that the argmin is never empty, as the 0-1 loss $\ell(f',(x,y)) = \frac{1}{n} \sum_{i=1}^n \mathbb{I}[f'(x_i) \ne y_i]$ always takes values in the finite set $\{0,1/n,2/n,3/n,\cdots,1\}$.)

Whenever there are multiple $f \in \cF$ that minimize $\ell(f,(x,y))$, our algorithm $A|(x,y)$ chooses the least element according to the well-ordering. In symbols, $A$ satisfies the following two properties, which also uniquely define it. $$\forall (x,y) \in (\cX \times \{0,1\} )^*  ~~ \forall h \in \cF ~~~ \left( \begin{array}{c} \ell(A|(x,y),(x,y)) \le \ell(f,(x,y)) \\ \wedge \\ \ell(A|(x,y),(x,y)) = \ell(f,(x,y)) \implies A|(x,y) \preceq f \end{array} \right).$$

By construction, our algorithm $A$ is an empirical risk minimizer. It only remains to prove that it satisfies the global consistency property. To this end, let $(x,y) \in (\cF\times \{0,1\})^n$ and let $x' \in \cX^m$ where $x'$ contains all the elements of $x$ (i.e., $\forall i \in [n] ~ \exists j \in [m] ~ x_i=x'_j$). Let $f=A|(x,y)$ and $f'=A|(x',f(x'))$. We must prove that $f'=f$.

By construction, the empirical loss of $f$ on the dataset $(x',f(x'))$ is $0$. Since $f'$ is the output of an empirical risk minimizer on the dataset $(x',f(x'))$, it too has empirical loss $0$ on this dataset. In particular, $f(x'_j)=f'(x'_j)$ for all $j \in [m]$. Moreover, since $A$ breaks ties using the ordering, we have $f' \preceq f$. However, since $f$ and $f'$ agree on $x'$, they also agree on $x$ and, hence, have the same loss on the dataset $(x,y)$ -- that is, $\ell(f',(x,y))=\ell(f,(x,y)) = \inf_{hf' \in \cF} \ell(f'',(x,y))$. This means that $A|(x,y)$ outputting $f$ implies that $f \preceq f'$. Thus we conclude that $f=f'$, as required.
\end{proof}

\commentout{\subsection{Proof of Theorem~\ref{th:compression-kl} (Compression Scheme Priors)}
\begin{proof}[Proof of Theorem~\ref{th:compression-kl}]
Let $W=W_2|(A_1|z)$ be a compression scheme prior and let $\langle \sz\rangle =\langle \sz_{\bf 0}, \sz_{\bf 1}\rangle$. We choose the conditional prior distribution as 
$$\pi(f|\langle\sz\rangle)=\frac{\sum_{z^k\in K(\sz)} \cP_{W_2|z^k}(f)}{\binom{2n}{k}},$$
where we denote by $K(z)$ the set of all subsets of $z$ of size $k$. Observe that $\pi$ is indeed an almost exchangeable prior. It holds that
\begin{align*}
\dkl{A|\sz_{\bf 0}}{\pi|\langle\sz_{\bf 0}, \sz_{\bf 1}\rangle}
& =\ex{f\sim A|\sz_{\bf 0}}{\log\frac{\cP_{A|\sz_{\bf 0}}(f)}{\pi(f|\langle\sz\rangle)}} \\
& =\ex{f\sim A|\sz_{\bf 0}}{\log\frac{\cP_{A|\sz_{\bf 0}}(f)\cdot \binom{2n}{k}}{\sum_{z^k\in K(\sz)} \cP_{W_2|z^k}(f)}} \\
& \leq \ex{f\sim A|\sz_{\bf 0}}{\log\frac{\cP_{A|\sz_{\bf 0}}(f)\cdot \binom{2n}{k}}{\cP_{W_2|(A_1|\sz_{\bf 0})}(f)}} \\
& = \ex{f\sim A|\sz_{\bf 0}}{\log\frac{\cP_{A|\sz_{\bf 0}}(f)}{\cP_{W|\sz_{\bf 0}}(f)}} +\log \binom{2n}{k} \\
& \leq \dkl{A|\sz_{\bf 0}}{W|\sz_{\bf 0}} + k\log(2n),
\end{align*}
where the first inequality holds since $A_1|\sz_{\bf 0}\in K(\sz)$ which implies that $\sum_{z^k\in K(\sz)} \cP_{W_2|z^k}(f)\geq \cP_{W_2|(A_1|\sz_{\bf 0})}(f)$.
The last inequality follows by the common bound $\binom{2n}{k}\leq (2n)^k$.
\end{proof}}

\subsection{Proof Sketch for Gibbs example}
\begin{proof}[Proof Sketch]
A known property of the $\hat\eta$-Gibbs algorithm (see for example \citep{GrunwaldM20}) relative to prior $W \mid \sz_{\bf 0}$ is that, among all learning algorithms $A$ that output a distribution on $\cF$, for all $\sz_{\bf 0}$ it achieves
\begin{equation}\label{eq:gibbsbound}
\min_A R(A| \sz_{\bf 0}; \sz_{\bf 0}) + \frac{\dkl{A | \sz_{\bf 0}}{W | \sz_{\bf 0}}}{n \hat\eta}. 
\end{equation}
Now assume that the prior $W$ is a compression scheme prior of some size $k$ and let $\pi | \langle \cdot \rangle$ denote the corresponding almost exchangeable prior satisfying (\ref{eq:compression}). 
If we consider the proof of Theorem~\ref{th:main} again, we see that if we set $A:= A_{\textsc{Gibbs}}$ to the Gibbs algorithm, and $A'$ to any other learning algorithm, then the crucial inequality (\ref{eq:finalesisumb}) in the proof of Theorem~\ref{th:main} still holds  
with $R(A |\sZ_{\bf 0}; \sZ_{\bf 0})$ on the right-hand side replaced by $R(A' |\sZ_{\bf 0}; \sZ_{\bf 0})$   and 
$\textsc{ub}$ set to $\dkl{A' | \sZ_{\bf 0}}{W| \sZ_{\bf 0}} + k \log 2n $.
Following all the remaining steps in the proof while keeping $\textsc{ub}$ in its new definition and keeping the distinction between $A'$ and $A$, we get the following corollary of Theorem~\ref{th:main}: {\em if we set $A$ to the Gibbs algorithm relative to size $k$ compression scheme prior $W$, and $A'$ to any other algorithm, we have, with the same abbreviations as in Theorem~\ref{th:main}}
\begin{multline*}
L(A_{\textsc{Gibbs}}| \sZ_{\bf 0} ; \cD) 
\stochleq^{\sZ_{\bf 0}}_{\eta} \\ 
L(A' | \sZ_{\bf 0} ;  \sZ_{\bf 0})  +  \minim{1}{2\beta} \cdot {R(A' | \sZ_{\bf 0} ;  \sZ_{\bf 0})} +  8 \cdot  \left(\frac{\dkl{A' | \sZ_{\bf 0}}{W|\sZ_{\bf 0} }+ O (k \log n) }{n\etamax}  \right)^{1/(2-\beta)}_{[**]} + \frac{6 \eta}{n}. \nonumber
\end{multline*}
\end{proof}

\end{document}